\documentclass{amsart}

\usepackage{lipsum}
\usepackage{amsfonts}
\usepackage{graphicx}
\usepackage{epstopdf}
\usepackage{algorithmic}
\usepackage{hyperref}
\usepackage{booktabs}
\usepackage{keyval}

\usepackage{xspace}
\usepackage{xcolor}
\usepackage{microtype}      
\usepackage{subcaption}
\usepackage{appendix}
\usepackage{algorithm}
\usepackage{cleveref}

\usepackage{array,multirow}
\usepackage{float}
\usepackage[export]{adjustbox}

\ifpdf
\DeclareGraphicsExtensions{.eps,.pdf,.png,.jpg}
\else
  \DeclareGraphicsExtensions{.eps}
\fi

\usepackage{amsmath,amsfonts,bm}

\def\eqref#1{equation~\ref{#1}}
\def\Eqref#1{Equation~\ref{#1}}

\def\1{\bm{1}}

\DeclareMathAlphabet{\mathsfit}{\encodingdefault}{\sfdefault}{m}{sl}
\SetMathAlphabet{\mathsfit}{bold}{\encodingdefault}{\sfdefault}{bx}{n}

\newcommand{\bfA}{{\bf A}}

\newcommand{\bfE}{{\bf E}}

\newcommand{\bfI}{{\bf I}}

\newcommand{\bfK}{{\bf K}}

\newcommand{\bfW}{{\bf W}}
\newcommand{\bfX}{{\bf X}}
\newcommand{\bfY}{{\bf Y}}

\newcommand{\bfa}{{\bf a}}
\newcommand{\bfb}{{\bf b}}
\newcommand{\bfc}{{\bf c}}

\newcommand{\bfs}{{\bf s}}
\newcommand{\bfx}{{\bf x}}
\newcommand{\bfy}{{\bf y}}

\newcommand{\bfq}{{\bf q}}

\newcommand{\bfn}{{\bf n}}

\newcommand{\bfr}{{\bf r}}

\newcommand{\bfz}{{\bf z}}

\newcommand{\hf}{{\frac 12}}
\newcommand{\grad}{{\boldsymbol \nabla}}

\newcommand{\bfepsilon}{{\boldsymbol \epsilon}}
\newcommand{\bftheta}{{\boldsymbol \theta}}

\newtheorem{theorem}{Theorem}[section]

\theoremstyle{definition}

\newtheorem{example}[theorem]{Example}

\theoremstyle{remark}

\numberwithin{equation}{section}

\begin{document}

\title{An Over Complete Deep Learning Method for Inverse Problems}



\author[Moshe Eliasof]{Moshe Eliasof$^{*}$}
\thanks{$^{*}$Department of Applied Mathematics and Theoretical Physics, University of Cambridge, Cambridge, UK. {\tt me532@cam.ac.uk }}

\author[Eldad Haber]{Eldad Haber$^{**}$}
\address{}
\thanks{$^{**}$Department of EOAS, University of British Columbia, Vancouver, BC, Canada. {\tt ehaber@eoas.ubc.ca}}

\author[Eran Treister]{Eran Treister$^\dag$}
\address{}
\curraddr{}
\email{}
\thanks{$^\dag$Department of Computer Science, Ben-Gurion University of the Negev, Beer Sheva, Israel.  {\tt erant@cs.bgu.ac.il}\\
This research was supported by the Israeli Council for Higher Education (CHE) via the Data Science Research Center at Ben-Gurion University of the Negev.}

\subjclass[2020]{68Q25, 68U05,68T09}

\keywords{Inverse Problems, Convolutional Neural Networks, Regularization}

\date{}

\dedicatory{}

\begin{abstract}
Obtaining meaningful solutions for inverse problems has been a major challenge with many applications in science and engineering. Recent machine learning techniques based on proximal and diffusion-based methods have shown promising results. However, as we show in this work, they can also face challenges when applied to some exemplary problems. We show that similar to previous works on over-complete dictionaries, it is possible to overcome these shortcomings by embedding the solution into higher dimensions. The novelty of the work proposed is that we {\bf jointly} design and learn the embedding and the regularizer for the embedding vector. 
We demonstrate the merit of this approach on several exemplary and common inverse problems.
\end{abstract}

\maketitle

\section{Introduction}
\label{sec1}

The solution of inverse problems entails the estimation of a model (solution) based on measured data, which is often noisy and ill-posed in nature \cite{taran, parker, somersallo, Tenorio2011}. 
These challenging problems arise in diverse fields such as geophysics \cite{taran}, astronomy, medical imaging \cite{nagyHansenBook}, and remote sensing \cite{vogel2002computational}. 
Given the ill-posed nature of the considered problems and the presence of noisy data, the application of regularization techniques becomes essential to achieve a stable and meaningful estimate of the model. Conventional regularization techniques involve using specific functions tailored to acquire desired properties, like Total-Variation \cite{RudinOsherFatemi92} or  $\ell_2$ smoothness \cite{Tenorio2011}. Alternatively, some methods leverage a-priori estimates of the model statistics, such as Gaussianity \cite{taran, somersallo}.

The art and sophistication in solving an inverse problem is to balance the {\em prior} knowledge about the solution and the {\em likelihood}, that is, the data fit of the predicted solution.
The solution is derived as a combination of regularization and data-fitting functions, and it represents a compromise between the data fit
and the prior. Traditionally, the prior is perceived to be less credible than the likelihood, as the likelihood is directly tied to the data of the problem to be solved.

In recent years,  machine learning has facilitated the development of data-driven regularization techniques and prior estimation. To this end,  supplementary data, aside from the measured data, contains many plausible solutions to the inverse problem. This additional data is then utilized to learn a regularization procedure, aiming to achieve superior results compared to traditional methods.
There are two main approaches for using available data for learning how to solve inverse problems.
The 
\emph{first} is problem-specific, that is, an end-to-end approach, where the regularization process is learned in the context of the particular inverse problem at hand. Such an approach was presented first in
\cite{haten} and then significantly improved in learning proximal maps by \cite{parikh2014proximal, mardani2018neural, jin2017deep, adler2017solving, mukherjee2021learning, drip} and reference within. The \emph{second} is learning a prior independently of the inverse problem, and then using the prior for the solution of the problem. This approach has been proposed in several recent works that utilize diffusion models to learn the prior \cite{chung2022diffusion, chung2022come, chung2022improving}.
Nonetheless, regardless of the approach used to learn the regularization function, in all the considered methods, the regularization is applied to the solution directly, that is, in its original coordinates. In other words, the regularization function uses the original properties and landscape of the solution space to measure its goodness. Therefore, the landscape of the regularization function may be highly non-convex and "unfriendly" to optimization procedures, especially those that use first-order methods such as gradient descent, with or without stochastic sampling nuances, such as Langevin dynamics. This is a well-known problem for optimization methods that operate on low dimensions in the original solution space (see \cite{taran,  nw}).

To this end, a technique that has been widely successful in the past was to embed the solution using an over-complete  dictionary (see \cite{chen2001atomic, CandesRombergTao2006, eladReview}
and references within). In this approach, one uses an over-complete dictionary and embeds the solution in higher dimensions than the original solution space. It is important to note that in this technique, the regularization is applied to the {\bf embedded solution vector} rather than the original solution vector. Canonical methods use $\ell_0$ and $\ell_1$ norms for the regularization of the embedding vector. These techniques have produced plausible and meaningful solutions, whether the dictionaries were learned or predefined, even though the regularization function was very simple, like the $\ell_1$ norm. We therefore propose to use similar concepts to inverse problem neural solution techniques. 

{\bf The contributions of this paper} are as follows: (i) We show how embedding-based techniques can be derived and learned in the context of contemporary data-driven regularization techniques. We show that by learning the embedding dictionaries {\bf and} the regularization function that operates on the embedded solution vector, one can obtain regularization functions that are "friendlier" to gradient-based optimization methods that are then utilized to solve the inverse problem at hand. (ii) 
Furthermore, we introduce two unrolled versions of the algorithm that can be interpreted as 
dynamical system in high dimensions that can bypass the highly nonlinear landscape of the problem in its original coordinates. 
Similar to other unrolled versions of an optimization process \cite{adler2017solving}, the unrolling allows for greater expressiveness and outperforms shared weights algorithms.
(iii) We give theoretical justification to the methods and show that while deep networks can be highly nonlinear with respect to the weights, it is possible to construct a {\em convex} functional that upon differentiation, leads to a deep neural network. Finally,
(iv) By examining several common inverse problems, we demonstrate that common architectures and approaches that use the original coordinates of the solution, can be significantly challenged while embedding-based techniques converge to meaningful solutions.

{\bf Connection to prior work:} Our method can be viewed as an extension of two popular and separate lines of techniques proposed for the solution of inverse problems. The first is using over-complete dictionaries, which was proposed in \cite{CandesRombergTao2006, chen2001atomic} and followed by many successful algorithms and implementations (see \cite{eladReview} and references within). Second, our work extends the incorporation of learning regularization priors \cite{haten, parikh2014proximal, mardani2018neural, jin2017deep, adler2017solving, mukherjee2021learning} by embedding the solution.
For learning the embedding, existing algorithms seek to find the optimal embedding over-complete dictionary (see
\cite{aharon2006k,aharon2006uniqueness,HoreshHaber2011,kasiviswanathan2012online,agarwal2014learning,huang2013optimal} and references within). In contrast, such embedding was not used in the context of learning regularization. Our work combines and extends both approaches by {\bf jointly} designing and learning an embedding and a regularization function in the high-dimensional embedded space. 

The rest of the paper is organized as follows. In Section~\ref{sec2}, we provide a mathematical background, accompanied by a  motivating example. In Section~\ref{sec3}, we reformulate the problem and show how it can be embedded in high dimensions and how such an embedding yields an easier problem to solve compared to the original problem. In Section~\ref{sec:architectures}, we propose architectures for the solution of the problem and discuss how to train the network. In Section~\ref{sec4}, we conduct several numerical experiments to demonstrate our approach, and 
Section~\ref{sec5} summarizes the paper.

\section{Mathematical Background and Motivation}
\label{sec2}

We first introduce the required mathematical background, followed by a simple, yet important example that demonstrates the shortcomings of existing inverse problem solution methods in deep learning frameworks.

\textbf{Problem formulation.} We consider the estimation of a discrete \emph{model} $\bfx \in \mathbb{R}^N$ from the measured data $\bfb \in \mathbb{R}^M$, and the relation between $\bfx$ and $\bfb$ is given by
\begin{eqnarray}
    \label{forprob}
    \bfA(\bfx) + \bfepsilon = \bfb.
\end{eqnarray}
Here, the forward mapping $\bfA:\mathbb{R}^N \rightarrow \mathbb{R}^M$ can be either linear or nonlinear. For simplicity, now consider linear inverse problems.
The vector $\bfepsilon$ is a noise vector that is assumed to be Gaussian with $0$ mean and $\sigma^2 \bfI$ covariance.
The forward mapping, $\bfA$, typically has a large effective null-space, which implies
that there are infinitely many models $\bfx$ that correspond to the same data, $\bfb$. 

\textbf{Traditional inverse problem solution methods.} We now provide a brief review of traditional estimation techniques for the model $\bfx$ given observed data $\bfb$, the forward mapping $\bfA$, and the statistics of the noise $\epsilon$. Let us first consider a Bayesian point of view for the recovery of the solution of the inverse problem. Assume that the model $\bfx$ is associated with a Gibbsian  prior probability density function $\pi(\bfx)$ of the form
\begin{eqnarray}
    \label{prior}
    \pi(\bfx) \propto  \exp\left(-R(\bfx) \right).
\end{eqnarray}
Then, the posterior distribution of $\bfx$ given the data $\bfb$ can be written as
\begin{eqnarray}
    \label{posterior}
    p(\bfx|\bfb) \propto \exp\left(-{\frac {1}{ 2\sigma^{2}}} \|\bfA \bfx - \bfb\|^2-R(\bfx) \right).
\end{eqnarray}

To obtain a solution (or a family of solutions), one may look at a particular procedure that uses the posterior. One popular approach is to use the Maximum A-Posteriori (MAP) \cite{degroot2005optimal, taran} estimate that maximizes the posterior by solving the optimization
problem
\begin{eqnarray}
    \label{xmap}
    \bfx_{\rm map} = {\rm arg}\min {\frac 1{2\sigma^{2}}} \|\bfA \bfx - \bfb\|^2 + R(\bfx).
\end{eqnarray}
The solution can be achieved by gradient descent iterations of the form
\begin{eqnarray}
    \label{xmapSolve}
    \bfx_{k+1} =  \bfx_k - \alpha \left(\sigma^{-2} \bfA^{\top}(\bfA \bfx_k - \bfb) + \grad_{\bfx} R(\bfx_k) \right).
\end{eqnarray}
Alternatively, it is possible to sample the posterior with some statistical sampling technique.
For instance, one can use Langevin dynamics \cite{pastor1994techniques}, to obtain a sampler of the form
\begin{eqnarray}
    \label{langSolve}
    \bfx_{k+1} =  \bfx_k - \alpha \left(\sigma^{-2} \bfA^{\top}(\bfA \bfx_k - \bfb) + \grad_{\bfx} R(\bfx_k) \right) + \sqrt{\alpha} \bfn,
\end{eqnarray}
where $\bfn \in N(0, \bfI)$ is a random variable. Also, we note that the use of Langevin dynamics is very popular in diffusion models \cite{yang2022diffusion, croitoru2023diffusion}.

The most common estimation or regularization approaches do not associate $R(\bfx)$ with the log of the prior, and use \Eqref{xmap} with some desired properties of the solution such as  low total-variation \cite{Tenorio2011}. 
By doing so, such traditional methods seek to balance between the prior and the likelihood. The regularization  $R(\bfx)$ is only approximately known,  and in many cases is heuristic-based.
Therefore, the solution $\bfx_{\rm map}$ or the samples obtained via Langevin dynamics from \Eqref{langSolve} represent a compromise between data fidelity (likelihood) that is obtained by minimizing $\| \bfA\bfx_k - \bfb \|{_2^2}$ and the prior incorporated by $R(\bfx_k)$. 
In particular, in most cases, for most traditional priors, the value of the prior probability at $\bfx_{\rm{map}}$ is small. That is, 
the solution to the inverse problem would not be a likely solution if we consider the prior alone. 
Recent techniques seek regions of agreement between the prior and the likelihood. Advances
in probability density estimation suggest that the regularization $R(\bfx)$ can be estimated
from data with greater accuracy compared to heuristic-based approaches such as TV priors, by utilizing a neural network (see \cite{yang2022diffusion,croitoru2023diffusion} and references within). This is a paradigm shift. It implies that we seek solutions $\bfx$ that
are significantly closer to the peak(s) of the prior, if they are to be realistic samples from the prior that also fit the data. 
As we see next, this makes the estimation of the model $\bfx$ substantially more difficult, because we need to derive algorithms that avoid local minima, and to find the global minima of the neural regularize $R(\bfx)$. 

We now provide an example that showcases our discussion above.

\begin{example}{\bf The duathlon problem.} 
\begin{figure}[h]
    \centering
    \includegraphics[width=12cm]{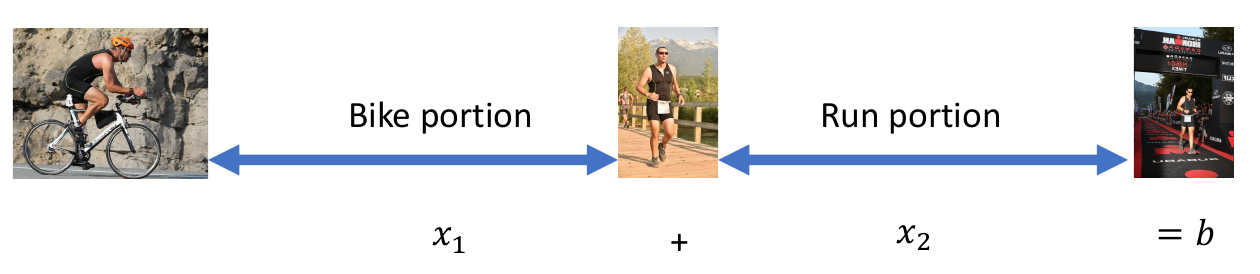}
    \caption{The duathlon problem, estimate the bike and run time from the total competition time.}
    \label{fig:figTri}
\end{figure}
Consider a duathlon that is composed of cycling and running segments.
Assume that we can measure the total time for an athlete to complete
the duathlon, but we are unable to see the time it takes her to finish a
particular segment. The question that we pose is, what was her time spent in each segment.

The mathematical model for this problem is a simple single linear equation, as shown in \Eqref{forprob} with
$\bfA = [1 \quad 1]$ and $\bfx = [x_1, x_2]^{\top}$, where $x_1$ is the time spent on cycling and $x_2$ is the time spent on running. Our goal is: given the single equation $x_1 + x_2 + \epsilon = b$, where only $b$ is available, and $\epsilon$ is a Gaussian noise, estimate $x_1$ and $x_2$. 
Without any prior information, it is impossible to estimate $x_1, x_2$ given the data $b$. However, if we observe previous dualthlete data, we are able to estimate the distribution of the times $x_1$ and $x_2$. Such a distribution is plotted in Figure~\ref{fig:fig_dist}(a). In our example, the distribution is made of $3$ groups. Assume that we measure a data point of an individual athlete (marked in yellow in Figure \ref{fig:fig_dist}(a)). We can try and estimate the individual's cycling and running times $x_1, x_2$ by sampling from the posterior. To this end we use the algorithms discussed in \cite{song2022solving, chung2022improving} 
that uses the prior within a diffusion process.
The result of this approach is presented in Figure~~\ref{fig:fig_dist}(b). 
\begin{figure}
    \centering
        \begin{tabular}{cc}
          \includegraphics[height=4cm]{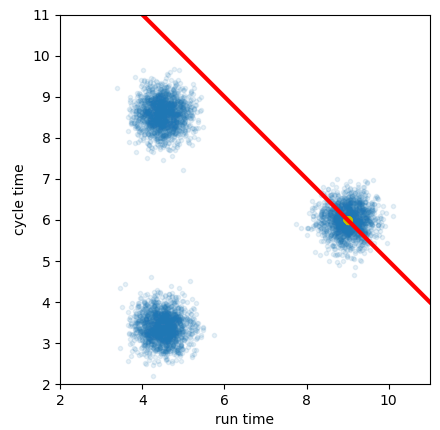}   &
          \includegraphics[height=4cm]{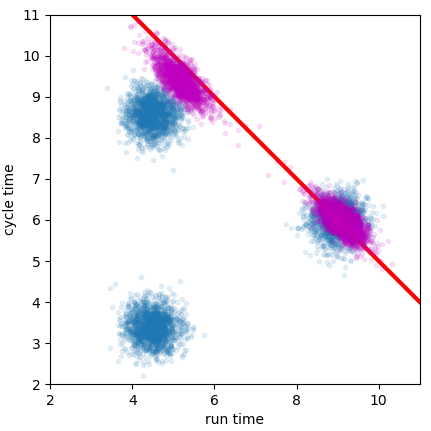} 
          \\
         (a) \small{Prior and likelihood} & (b) 
         \small{Sampling  the Posterior}  
        \end{tabular}
    \caption{Experiments with the dualthlon problem. The prior is made of three Gaussians and the data (yellow point) are presented on the left along with all possible points that fit the data (red line). Recovering a model given the data by sampling the posterior using diffusion is presented on the right (magenta points).}
    \label{fig:fig_dist}
\end{figure}
The sampling from the posterior contains two main groups. However, while it is evident that one group is realistic (and indeed the true solution is sampled from this group) the second group is highly unlikely. This is because it does not coincide with the points of the prior. In fact, the points in this group are of very little probability to occur.
The two groups we observe represent two local minima of the posterior where the one on the right is also the global minima.
Nonetheless, starting at many random points, stochastic gradient-based optimization for the posterior is likely to converge to both global and local minima, unless the Langevin dynamics is applied for a very long time. Furthermore, since diffusion-based algorithms typically avoid computing the probability and compute the score (that is, $\grad \log \pi(x)$) instead, it is impossible to quantify the posterior probability of each point that is obtained by a diffusion process.

We thus observe that in this very simple case, diffusion models and other models based on optimization can yield unsatisfactory results. We have observed similar problems for much larger and more realistic problems, such as the estimation of susceptibility from magnetic data and image deblurring with very large point spread functions. For some of these problems, many local minima were observed and it was impossible to find samples that are close to the global minima.
\end{example}

The problem above demonstrates two key shortcomings of existing approaches to solving inverse problems using deep learning frameworks, assuming sufficient reliable training data is available:
\begin{itemize}
    \item The prior can be estimated well and we expect that the posterior will overlap some parts of the prior. Therefore, we seek points that are close to the {\bf global} minima of the posterior. 
    \item Finding a global minimum is a very difficult problem. Nonetheless, with the emergence of complex priors (such as diffusion models), such regularization leads to highly nonconvex problems where many local minima typically exist. 
\end{itemize}
Given these shortcomings, the question is, can we derive an algorithm that can be applied in the face of highly non-convex optimization problems?
We now discuss a methodology that allows us to obfuscate these difficulties.

\section{Reformulating the Solution of Inverse Problems by Embedding and deep learning}
\label{sec3}

As previously discussed, the main issue in solving the inverse problem, both in terms of sampling and in the MAP estimation
is that we face a global optimization problem with many local minima. We now show that by 
reparametrizing the problem and embedding it in higher dimensions it is possible to obtain a more convex problem that is easy to work with and sample
from, and therefore to find more likely solutions to the inverse problem.

\subsection{High Dimensional Solution Embedding}

Let us consider an embedding of the solution $\bfx \in \mathbb{R}^N$ into a vector
$\bfz \in \mathbb{R}^K$ where  $N \ll K$, with an embedding matrix $\bfE:\mathbb{R}^K \rightarrow \mathbb{R}^N$, that is,
\begin{eqnarray}
    \label{embed}
    \bfx = \bfE \bfz
\end{eqnarray}
The vector $\bfz$ is the new variable we work with, and we solve the inverse problem
with respect to $\bfz$ rather than $\bfx$. In what follows, we will learn an embedding and a regularization function that operates on $\bfz$ such that the resulting optimization problem is more attainable to numerical treatment compared with the original one for  $\bfx$.

As we discussed in the introduction, the idea of embedding the solution in a larger space and regularizing the embedding vector has been thoroughly studied and discussed in the field of inverse problems \cite{CandesRombergTao2006} in the context of $\ell_1$ regularization and basis pursuit \cite{chen2001atomic}.
In this case, one replaces the original problem with the $\ell_1$ regularization
\begin{eqnarray}
    \label{zmap}
    \bfz_{\rm map} = {\rm arg}\min {\frac 1 {2\sigma^{2}}} \|\bfA \bfE \bfz - \bfb\|^2 + \gamma\|\bfz \|_1
\end{eqnarray}
which is associated with the prior density  $\pi(\bfz) \propto \exp(-\gamma \|\bfz \|_1)$.
The density in this case is log-concave and, hence, robust convex optimization algorithms
can be used to solve the problem. The complicated part now is to choose
an appropriate embedding matrix $\bfE$. As discussed in the introduction, the main focus of this line of work was to learn an appropriate 
embedding assuming the $\ell_1$ prior.
An extension of \Eqref{zmap} is to  {\bf jointly} learn the embedding matrix $\bfE$ {\bf and} a  regularization function. Furthermore,  we propose an unrolled version of this process that yields a neuro-ordinary differential equation \cite{HaberRuthotto2017, E2017,chen2018neural}.

\subsection{Learnable  Embedding and 
Regularization  in high dimensions}
\label{sec:learnableEmbeddingAndPotential}

\Eqref{zmap} uses a high dimensional embedding, and employs the $\ell_1$ norm as a regularization for 
$\bfz$. However, one can learn a regularization function
$\phi(\bfz, \bftheta)$, with parameters $\bftheta$. This leads to a minimization problem of the form
\begin{eqnarray}
    \label{zmape}
    \bfz_{\rm map} = {\rm arg}\min_{
    \bfz 
    } {\frac 1 {2\sigma^{2}}} \|\bfA \bfE \bfz - \bfb\|^2 + \phi(\bfz,\bftheta)
\end{eqnarray}
By carefully learning both $\bfE$ and $\phi$, we can obtain an optimization problem with favorable properties. 
To understand the rationale behind this we first review the well-known mountain pass theorem that states
\begin{theorem}[Mountain pass]Let $f(\bfx)$ be a bounded function from $\mathbb{R}^N$ to $\mathbb{R}$ with continuous second derivatives.
Let $\bfx_1$ and $\bfx_2$ be local minima of the function, that is
$$ \grad f(\bfx_i) = 0\quad {\rm and} \quad  \grad^2 f(\bfx_i)>0, \quad i=1,2.$$

Then, there exists a path (a mountain pass) defined by
\begin{eqnarray}
\label{mp}
&& \bfx(t) = \bfx_1 + \bfs(t) \\
\nonumber
&&  \bfs(0) = 0, \quad \bfs(1) = \bfx_2-\bfx_1
\end{eqnarray}
such that the function on this path
$ f(\bfx(t))$
have a unique maximum.
\end{theorem}
The mountain pass theorem is a simple extension of Rolle's theorem in basic calculus.
It states that in order to move from one minima of a function to another one has to climb a hill.

The mountain pass theorem \cite{strang1991calculus} assumes that we are limited to the topography of the function $f(\bfx)$.
Increasing the dimension of the problem, we are able to generate a new topography that is favorable to the optimization problem at hand.

This is motivated by the following theorem:
\begin{theorem}[The Mountain Bypass]
\label{theorem:bypass} .
Let $\bfx_1$ and $\bfx_2$ be local minima of $f(\bfx)$
with $f(\bfx_2) \le f(\bfx_1)$.
Assume  some  embedding of the form
$$ \bfx = \bfE \bfz $$
where $\bfE$ is an $n \times k$ full rank matrix with $k>n$.
Assume also that there is a function, $g:R^n \rightarrow R^k$ such that $g(\bfx) = \bfz$,  that uniquely maps the vectors $\bfx \in R^n$ onto
a subspace of the vectors in $R^k$.

Finally, consider a function $\Psi:R^k \rightarrow R$ such that
$$ \Psi(g(\bfx)) = f(\bfx)$$

Then, there exist an embedding $\bfE$ and a function $\Psi$ such that if 
$\bfx_1$ and $\bfx_2$ are local minima of $f$
and $\bfz_1 = g(\bfx_1)$ and $\bfz_2 = g(\bfx_2)$
then there is a continuous path $\bfz(t), 0\le t \le 1$ such that
$\bfz(0) = \bfz_1$, $\bfz(1) = \bfz_2$ and
$$ {\frac {d\Psi(\bfz(t))}{dt}} \le 0, \quad 0 \le t \le 1 $$
That is, there is a path $\bfz(t)$ that bypasses the mountain pass when we are unlimited by the topography given in the space spanned by $\bfx$.
\end{theorem}

\begin{proof} 
We prove the theorem by a simple construction.
We choose $g(\bfx) = \bfX \bfx$
where $\bfX = \bfE^{\dag}$.
We can then decompose $\bfz$ into two parts, one in the active space of $\bfE$ and one in the null space of $\bfE$, that is
\begin{eqnarray}
    \label{decomp}
    \bfz = \bfX \bfx + \bfY \bfy
\end{eqnarray}
where the matrices $\bfX$ and $\bfY$ are the active and null space of $\bfE$, that is
$$ \bfE \bfX = \bfI  \quad \bfE \bfY = 0 \quad {\rm and} \quad  \bfY^{\top}\bfY = \bfI$$
Given a vector $\bfz$ one could compute $\bfx$ and $\bfy$ by $\bfx = \bfE \bfz$ and $\bfy =  \bfY^{\top}\bfz$.

We then propose to construct $\Psi$ from two parts
\begin{eqnarray}
\Psi(\bfz)  = f(\bfx) + \Omega(\bfy)
= f(\bfE\bfz) + \Omega(\bfY^{\top}\bfz)
\end{eqnarray}
where $\Omega(0) = 0$, is a continuously differentiable function that we can choose as we please.

Let $\bfz_1 = \bfX\bfx_1$ and
$\bfz_2 = \bfX \bfx_2$. Clearly, neither
$\bfz_1$ or $\bfz_2$ has a component in the null space of $\bfE$.

To build a path from $\bfz_1$ to $\bfz_2$ we first 
consider the part in the active space of $\bfE$.
We set the active part of the space as the mountain pass \Eqref{mp}
$$ \bfx(t) = \bfx_1 + \bfs(t) $$
Next, we consider a path in the orthogonal part of $\bfE$ and let
$$\bfy(t) = t(1-t) \bfy$$
for some $\bfy \not=0$.

We then have that
$$ \Psi(\bfz(t)) = f(\bfx(t)) + \Omega(\bfy(t)).$$
Since we are free to choose $\Omega$ as we please we choose it such that $\Psi(\bfz(t))$ is monotonically decreasing for $0 \le t \le 1$. 
\end{proof} 

The mountain overpass theorem is not useful in the context of optimization since it requires the knowledge of the point that minimizes $f(\bfx)$
in order to design the path. Nonetheless, the theorem is very useful in the context of machine learning. In this context, we assume to have a family of functions $f(\bfx, \bfb)$ where $\bfb$
is our data vector, and that the solutions for different $\bfb$'s are known. Furthermore, in many cases, the points that minimize $f$ for different $\bfb$ cluster together. Learning a map that yields a path from a starting point to the appropriate solution is therefore the goal of the learning.

In the context of solving the inverse problem, the original problem with $\bfx$ is changed to the problem
$$ \min_{\bfz} \hf \|\bfA \bfE \bfz - \bfb\|^2 + R(\bfE\bfz) + \Omega(\bfY^{\top} \bfz) $$
The addition of a learnable function $\Omega$ is the one that allows us to move from one minimum to the next without the need to go through the mountain pass.
In the learning process proposed above, rather than learning $R$ and then $\Omega$ we choose a function $\phi(\bfz,\bftheta)$ that yields a smooth path from the starting point $\bfz_1$ to the desired point, $\bfz_2$ that is the solution to the inverse problem.

\bigskip 

We now demonstrate the importance of this Theorem using a simple example.
\begin{example}{\bf Pointwise recovery (denoising) with a double potential prior} \\
Assume that $x \in \mathbb{R}$, and the forward mapping is the identity. That is, $b = x + \epsilon$, where 
$\epsilon$ is a small Gaussian noise. This is the simplest 1-dimensional denoising problem. Assume that the prior of $x$ is a double-well potential of the form 
$$\pi(x) \propto \exp\left(- \gamma^{-1}(x-\mu)^2 (x+\mu)^2 \right).$$
This potential is plotted in Figure~\ref{fig:phi2d}(a).
\begin{figure}
    \centering
        \begin{tabular}{cccc}
          \includegraphics[height=2.6cm]{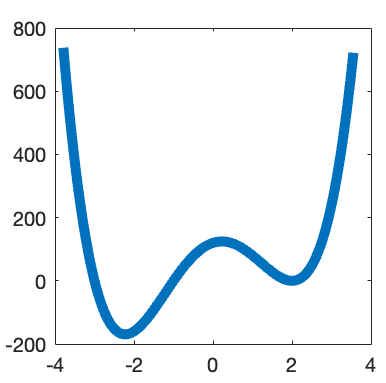}   &
          \includegraphics[height=2.6cm]{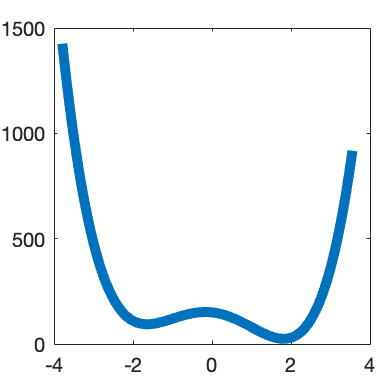}   &
          \includegraphics[height=2.62cm]{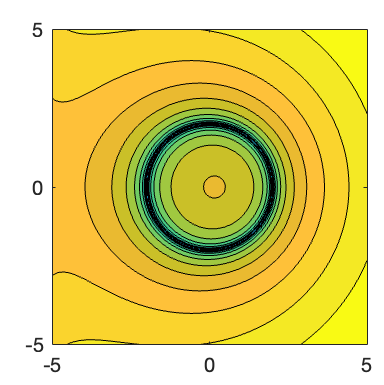}   &
          \includegraphics[height=2.62cm]{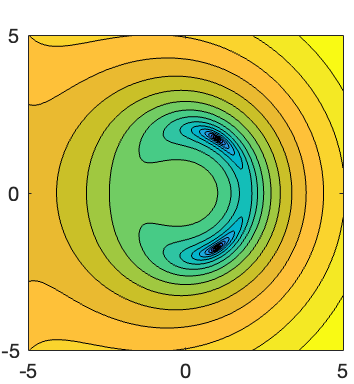}   
          \\
         (a) \small{Prior in $x$} & (b) \small{Posterior in $x$} & (c) \small{Prior in $\bfz$.} & (d) \small{Posterior in $\bfz$.}  
        \end{tabular}
    \caption{A non-convex prior with 2 local minima in $x$ is replaced with a learned quasi-convex prior in higher dimensions. Plots (c) and (d) are in log scale. }
    \label{fig:phi2d}
\end{figure}
Given data $b$, for the MAP estimator, one needs to minimize the log posterior distribution and solve the optimization problem
$$ \min_{x} {\frac 1{2\sigma}}(x-b)^2 + \gamma^{-1}(x-\mu)^2 (x+\mu)^2.  $$
This is clearly a non-convex problem. The negative log posterior is plotted in Figure~\ref{fig:phi2d}(b). This problem has two local minima. One close to $\mu$ and the other close to $-\mu$. A gradient descent type algorithm would therefore converge to one of the two minima, depending on the starting point. 

Now, consider the simple embedding
$x = \bfE \bfz$
where $\bfE = [1,0]$ and $\bfz = [z_1, z_2]^{\top}$.
Let us explore \Eqref{zmape} where we can learn or choose
a function $\phi(\bfz)$ as we please. One such function is
$$ \phi(\bfz) = (z_1^2 + z_2^2 -\mu^2)^2 = \left( (z_1 - \mu)(z_1 + \mu) + z_2^2 \right)^2 $$ 
that is plotted in Figure~\ref{fig:phi2d}(c).
Note that for $z_2=0$ the prior is reduced to the original prior.
The function in 2D has a clear path that connects both minima in 1D. The posterior surface of $\bfz$ given the data is plotted in Figure~\ref{fig:phi2d}(d). The function also has 2 minima however, they are benign since both of them have the same $z_1 = x$ and therefore
upon solving this problem with a gradient descent method, we can obtain the unique and correct minima.
\end{example}

\subsection{From Optimization to Network Architectures}

There are a number of options to use the optimization problem \eqref{zmape} and to generate 
a deep network. 
We materialize a simple descent of \Eqref{zmape} yielding a network of the form:
\begin{eqnarray}
\label{fdynamics}
{\bf OPTEnet}: \quad \bfz_{j+1} &=& \bfz_j - h_j \bfE^{\top}\bfA^{\top}(\bfA \bfE \bfz_j -  \bfb) - h_j \grad_{\bfz} \phi(\bfz_j ; \bftheta) 
\end{eqnarray}
where $h_j$ is a non-negative step size. 
We name this network OPTEnet since it evolves from an \textit{optimization} setting on the hidden Embedded variable $\bfz$. The method bears similarity to the method proposed by \cite{jin2017deep} for solving inverse problems. However, it differs from \cite{jin2017deep} in its utilization of a learnable embedding matrix $\bfE$. Note,
that the network has a single  learnable embedding matrix $\bfE$ and a single
potential layer $\phi(\cdot ; \bftheta)$, that is shared across all layers, parameterized by the weights $\bftheta$.

However, using a single embedding and shared parameters $\bftheta$ may yield a network with limited expressiveness \cite{ongie2020deep}, as it is shared across all layers. 
One way to increase the expressiveness of the network is to unroll the iteration \cite{ongie2020deep}, effectively changing
both $\bfE$ and $\bftheta$ with the iteration $j$, obtaining:
\begin{eqnarray}
\label{fdynamicsu}
{\bf EUnet}: \quad \bfz_{j+1} &=& \bfz_j - h_j \bfE_j^{\top}\bfA^{\top}(\bfA \bfE_j \bfz_j -  \bfb) - h_j \grad_{\bfz} \phi(\bfz_j ; \bftheta_j) 
\end{eqnarray}
We call this network EUnet as it includes an \textit{unrolling} and embedding steps.
This network
extends the idea of embedding beyond a single embedding matrix. 
While \Eqref{fdynamics} is a gradient flow step for the optimization problem in \Eqref{zmape}, its unrolled variant in \Eqref{fdynamicsu} is not.  

Even though \Eqref{fdynamicsu} is not a solution to an optimization problem, using some mild assumption it does converge to
a continuous time-dependent process. Such a process  allows us to consistently change the number of layers (time steps), as has been shown by \cite{chen2018neural, ANODEV2, haber2019imexnet}.  
Assuming that $\bfE_j=\bfE(t_j)$ and $\phi(\bfz_j \bftheta_j) = \phi(\bfz(t_j) \bftheta(t_j))$ are piece-wise smooth in time, one way to interpret the method is as a time discretization, where the network layer index $j$ is interpreted as the $j$-th time step, of the dynamical system:
\begin{eqnarray}
    \label{ddynamics}
    {\frac {d\bfz}{dt}} &=& \bfE(t)^{\top}\bfA^{\top}(\bfb - \bfA \bfE(t) \bfz(t)) - \grad_{\bfz} \phi(\bfz(t) ; \bftheta(t), t) \quad t \in [0,T] \\
    \bfz(0) &=& \bfz_0 
\end{eqnarray}
These types of systems arise in instantaneous control when the 
objective function changes in time \cite{harrison1983instantaneous}. 

The two discrete models {\bf OPTnet} and
{\bf EUnet} have very different properties when 
the number of layers increases. In {\bf OPTnet}
the layers are added to obtain a steady state, however, in {\bf EUnet} we integrate over a fixed time $[0,T]$, and therefore when adding layers we need to decrease $h$ and maintain smoothness.

\textbf{Stochastic sampling.} Note that \Eqref{fdynamics}, and \Eqref{fdynamicsu}
 are deterministic discretizations of differential equations. However, similarly to \cite{martin2012stochastic, yang2022diffusion, croitoru2023diffusion}, it is possible to augment \Eqref{fdynamics}, and \Eqref{fdynamicsu} with a stochastic sampling mechanism as shown in \Eqref{langSolve}, yields a stochastic differential equation. This extension of our model is left for future work.

\section{Network architectures and training}
\label{sec:architectures}

\subsection{Architectures}

The systems in Equations \ref{fdynamics} and \ref{fdynamicsu} require the utilization of a gradient of the potential function $\phi$ with respect to $\bfz$. Note that in contrast to any other inversion technique known to us, our $\phi$ operates
on the higher dimension $\bfz$ rather than the lower dimension $\bfx$. Also, below we select the embedding matrices $\bfE$.

While it is possible to use many different architectures for the potential our basic form for the potential $\phi$ is a simple network of the form
\begin{eqnarray}
    \label{potential}
    \phi(\bfz, \bftheta, t) = \bfq^{\top}\sigma_I\left(\bfK_n \sigma_I(\bfK_{n-1} \sigma_I( ...\,  \bfK_2\sigma_I(\bfK_1\bfz + t\bfb_1) +\bfb_2) \ ... \right)
\end{eqnarray}
Here, the learnable parameters are $\bftheta =\{ \bfK_1,\ldots, \bfK_N, \bfb_1, \ldots, \bfb_N, \bfq\}$, where $\bfb_j$ are the time embedding vectors. The activation function
$\sigma_I$ are chosen as the integral of common activation 
functions. For example, in our experiments, we choose the integral of the hyperbolic tangent ($\rm{tanh}$) activation function:
$$ \sigma_I(t) = \int \tanh(t)\, dt = \log (\cosh (t)).$$

Now assume that we use a potential composed with single Layer in {\bf OPTnet} and we choose $\bfK = \gamma \bfI$ with $\gamma \gg 1$. We then have that
$$ \phi(\bfz) = \log (\cosh (\gamma \bfz)) \approx \gamma \|\bfz\|_1. $$
Thus, the method allows us to recreate the standard over-complete dictionary.
However, it is interesting to see that we are able to obtain more interesting surfaces.
Realizations for such surfaces are plotted in Figure~\ref{fig:surfs}.
\begin{figure}
    \centering
    \begin{tabular}{ccc}
    \includegraphics[width=4.5cm]{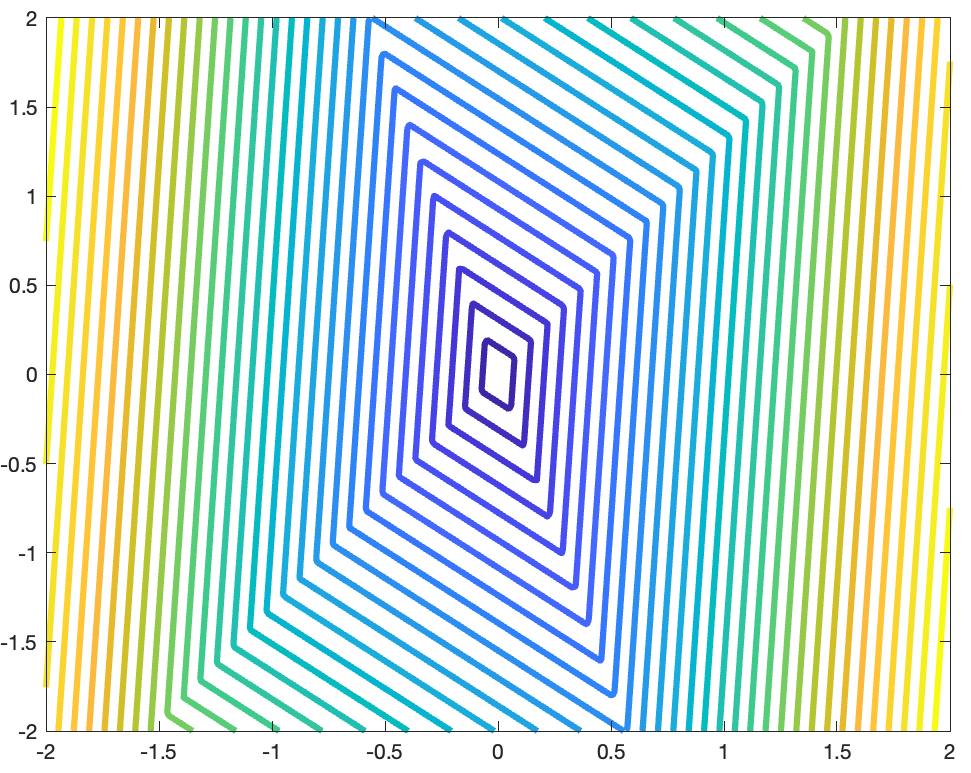} &
    \includegraphics[width=4.5cm]{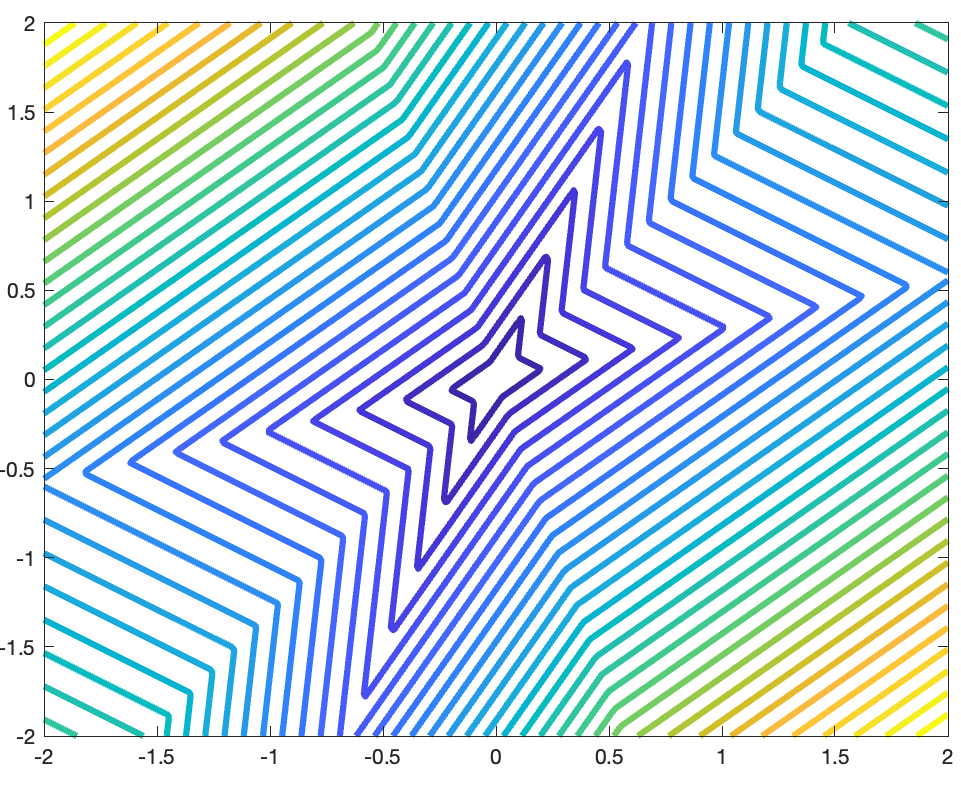} &
    \includegraphics[width=4.5cm]{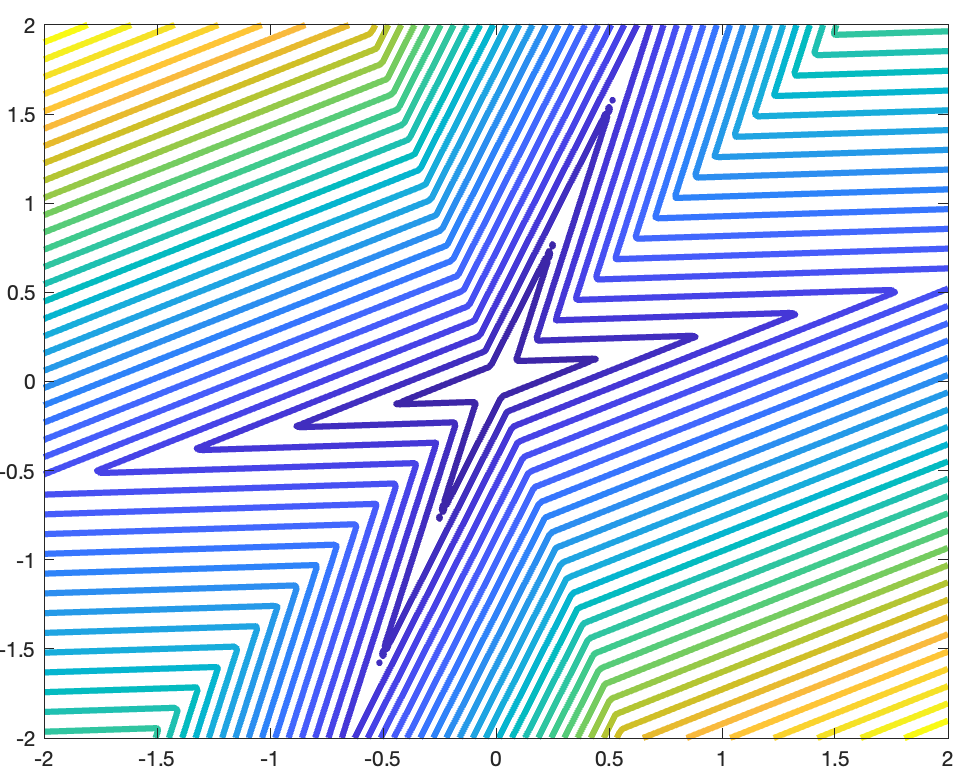} \\
    1 Layer & 2 Layer &  3 Layer 
    \end{tabular}
    \caption{The potential function $\phi$ with a different number of layers.}
    \label{fig:surfs}
\end{figure}

Note that upon differentiating with respect to $\bfz$ we obtain (neglecting the time embedding for brevity)
\begin{eqnarray}
\label{eq:network_equation}
\grad_{\bfz} \phi(\bfz, \bftheta) &=& 
 \bfK_1^{\top} {\rm diag}(\sigma(\bfK_1\bfz ) )\bfK_2^{\top} {\rm diag}(\sigma(\bfK_2\bfz) )\cdots\\ && \cdots\,
 \bfK_{n-1}^{\top} {\rm diag}(\sigma(\bfK_{n-1} \bfa_{n-1}))\bfK_n^{\top} {\rm diag}(\sigma\left(\bfK_n \bfa_n \right)) \bfq
 \\
 \bfa_j &=&  \sigma_I(\bfK_{j-1} \sigma_I( \cdots\,  \bfK_2\sigma_I(\bfK_1\bfz) 
\end{eqnarray}

In Equation \ref{eq:network_equation}, $\bfK_i$ are trainable convolution filter matrices.

In order to learn the embedding matrix $\bfE$
we use a simple network that takes $t$ and maps it
into a kernel using a 2-layer network, that is
\begin{eqnarray}
    \label{Et}
    \bfE(t) =  \bfW_2 \sigma(t \bfW_1 + \bfc)
\end{eqnarray}
Here $\bfW_1, \bfW_2$, and  $\bfc$ are trainable weight matrices that determine the embedding as a function of time. 

Note that if in \Eqref{potential} we reduce $\bfE$ to the identity matrix and choose $\grad \phi$ to be a Unet then we obtain an architecture similar to the one proposed in \cite{adler2017solving}. Thus, our architecture allows for the combination and extension of techniques that range from a classical $\ell_1$ with dictionaries to deep networks with a learned regularization.

\subsection{Training }

Throughout the experiments in Section \ref{sec4}, we aim to train the 'Proximal', OPTEnet, and EUnet to reduce the residual, i.e., to decrease the data fit error. Denoting the noisy data by $\bfb$, and the true solution as $\bfx$, we aim to minimize the empirical risk:
\begin{equation}
    \label{eq:loss}
    \mathcal{L} = {\mathbb{E}}\| f(\bfb) - \bfx \|_2,
\end{equation}
where $f(\cdot)$ is the respective neural network. 

We train each network for 300 epochs, and determine the hyperparameters by a grid search. Our hyperparameters are the learning rate $\rm{lr} \in \{1e-2,1e-3,1e-4,1e-5\}$, the weight decay $\rm{wd} \in \{1e-3,1e-4,1e-5,1e-6,0\}$ and batch size $\rm{bs} \in \{32,64,128,256,512\}$.

To train the diffusion model \cite{chung2022diffusion}, we used the code provided by the authors. Note that as per the work suggested in \cite{chung2022diffusion} a major difference between the diffusion model and 'Proximal' \cite{mardani2018neural} and our OPTEnet and EUnet, is that the diffusion model is not trained in an end-to-end manner for the specific inverse problem we aim to solve. 
As can be seen in our experiments, the approach of learning a 'global' prior via a diffusion model, may lead to sub-par results on hard inverse problems, where the forward problem is not close to denosing.

\section{Numerical Experiments}
\label{sec4}

As we see next, the methods proposed in this work perform
well for classical problems such as tomography and image deblurring, similarly to other existing methods.  More importantly, our proposed method significantly outperforms existing methods
for highly ill-posed problems such as the inversion of magnetic data.
In our experiments, we use we use the MNIST \cite{mnistlecun2009}, and STL10 \cite{coates2011analysis} datasets.

We experiment with the two variants of the embedded solution proposed in this paper, namely ${\bf OPTnet}$ and ${\bf EUnet}$. As a comparison, we consider a diffusion
model applied to an inverse problem as proposed by \cite{chung2022diffusion} that is denoted by 'Diffusion' throughout this section, and the Unrolled proximal iteration proposed by \cite{mardani2018neural}, similarly denoted by 'Proximal'.


\subsection{The dualthlon problem}

In Section~\ref{sec2} we have seen how diffusion models
converge to two different local minima in Figure \ref{fig:figTri}(b) - one which is local
and is unrealistic and one is global and is therefore
desired. 
As can be seen in Figure~\ref{fig:figTri}(c), using the same data with our $\bf{OPTEnet}$,  we obtain a sampling of the correct minima.

We now do a similar test on a large number of points.
We first train the networks 
$\bf{OPTEnet}$ and $\bf{EUnet}$ to solve the problem. The embedding dimension of $\bfz$ here is $128$. Thus, it is 64 times larger than the size of the input $\bfx$.
We also train a network based on proximal methods \cite{mardani2018neural} where no embedding is used for comparison.

We generate our training and validation data by sampling 10,000 data points chosen from $3$ Gaussians with means $\mu_i, i=1,...,3$,, that is, $\bfx \sim N(\mu_i, \bfI)$.
For each sampled point $\bfx = (x_1, x_2)^{\top}$ we generate the measured data by the summation $x_1+x_2$, and add $1\%$ noise, attempting to recover $\bfx = (x_1,x_2)^{\top}$.

A table with the mean-squared-error (MSE) for each method, on the final data set
is presented in Table~\ref{tabTri}. 
We find that the proposed 
$\bf{OPTEnet}$ and $\bf{EUnet}$ architectures, that use embedding of the solution in a larger space, perform significantly better than the inversion based on the proximal method --  that solves the problem in the original space. Among the two variants, the Unrolled method EUnet, performed significantly better compared with the optimization-based network OPTEnet. 
\begin{table}[h]
\centering
\begin{tabular}{ccccc} \toprule
{\bf Method} & Proximal    &  OPTEnet  & EUnet   \\ \hline
{\bf MSE} & $3.5 \times 10^{-1}$  & $8.2 \times 10^{-2}$ &
 $4.1 \times 10^{-2}$ \\
 \bottomrule
\end{tabular}
\caption{Mean-Square-Error on validation data for the duathlon problem \label{tabTri}}
\end{table}

\subsection{Image deblurring}

Image deblurring is a common inverse problem where the forward problem is given by the integral (see \cite{nagyHansenBook}
\begin{eqnarray}
    \label{blurring}
    \bfb(\bfr) = \int_{\Omega} K(\bfr-\bfr') \bfx(r') d\bfr'
\end{eqnarray}
where $K(\bfr-\bfr')$ is a point spread function (PSF) with the form $K(\bfr-\bfr') = \exp\left(-s^{-1}\|\bfr -\bfr'\|\right)$. Here $s$ plays the role of smoothing. For a small $s$ the data is very similar to the original image and the problem is almost well posed, while for a large $s$ the data is highly smoothed and the problem is highly ill-posed.
We use the algorithm presented in \cite{nagyHansenBook} and discretize the integral on a $96 \times 96$ grid. 
We then train our network as well as the network \cite{adler2017solving} on image deblurring problems where 
the blurring kernel changes from light to heavy blurring. We also use a trained diffusion model as proposed in \cite{song2022solving} on recovering the original image.
The results are summarized in Table~\ref{tab2}.
\begin{table}[]
    \centering
    \begin{tabular}{c|c|c|c}
        \toprule
      Blurring Kernel size $s$   & Diffusion & Proximal & EUnet (Ours) \\ \midrule
       1  &   4.3e-3 & 5.6e-3   &    1.3e-3\\ 
       3  &  4.3e-2 & 3.0e-2    &2.3e-2\\ 
       5  &  1.9e-1 & 4.7e-2   & 4.2e-2   \\ 
       7  & 5.0e-1  &   6.3e-2   &  5.7e-2         \\         
       9  & 9.5e-1  &  8.9e-2   &  7.1e-2  
       \\ \bottomrule
    \end{tabular}
    \caption{Recovery loss (MSE) of the  STL-10 test set, of different methods and different blurring kernels sizes $s$ .}
    \label{tab2}
\end{table}

 \begin{figure}[t]
\begin{center}
\resizebox{1\linewidth}{!}{
\begin{tabular}{c|c|c|c|c|c}
\centering \rotatebox{90}{$s=1$} & 
\includegraphics[width=2.25cm,valign=c]{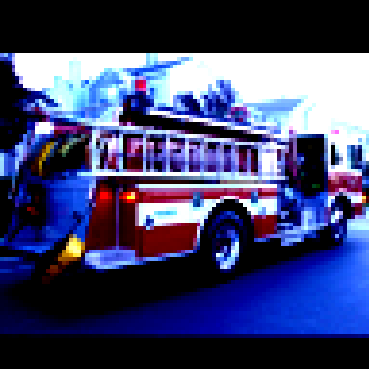}  &
\includegraphics[width=2.25cm,valign=c]{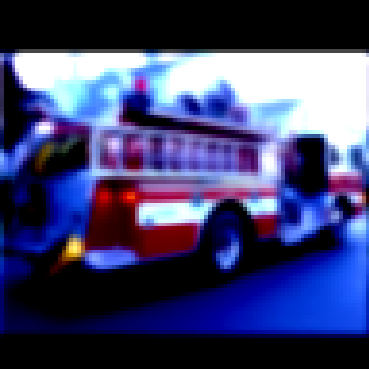} &
\includegraphics[width=2.3cm,valign=c]{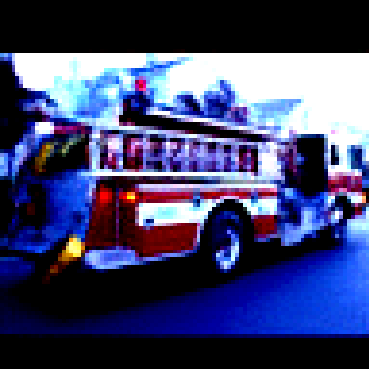}  &
\includegraphics[width=2.20cm,valign=c]{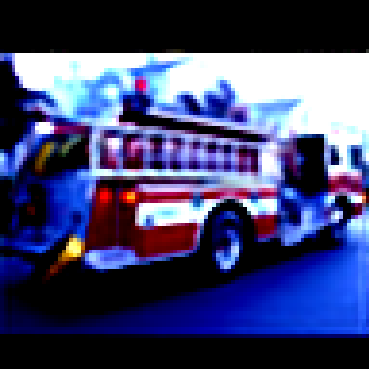} &
\includegraphics[width=2.25cm,valign=c]{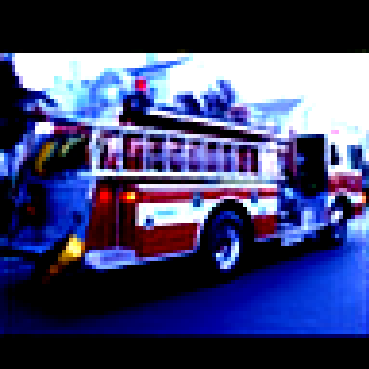}
\\ \rotatebox[origin=t]{90}{$s=3$} & \includegraphics[width=2.25cm,valign=c]{images/s1_img5_gt.png} &
\includegraphics[width=2.25cm,valign=c]{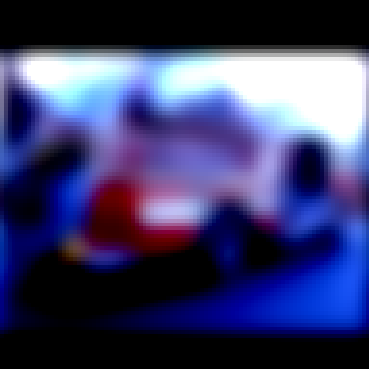} &
\includegraphics[width=2.25cm,valign=c]{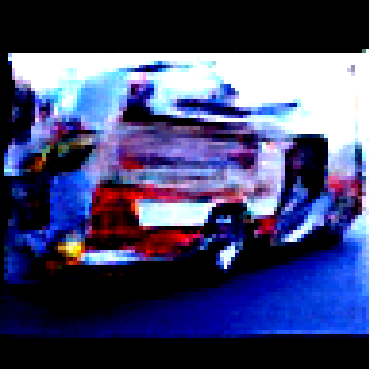} &
\includegraphics[width=2.25cm,valign=c]{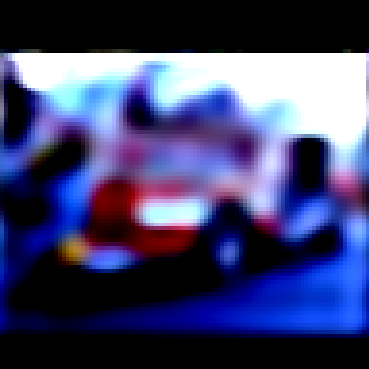} &
\includegraphics[width=2.25cm,valign=c]{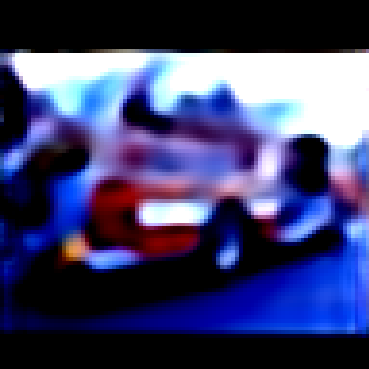}
\\ \rotatebox{90}{$s=5$} &  \includegraphics[width=2.25cm,valign=c]{images/s1_img5_gt.png}  &
\includegraphics[width=2.25cm,valign=c]{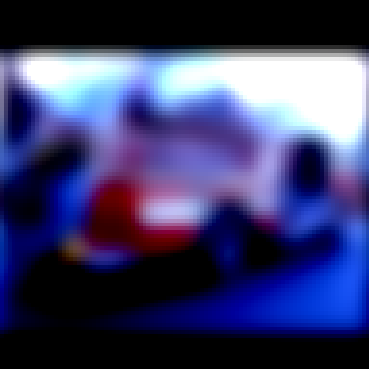} &
\includegraphics[width=2.25cm,valign=c]{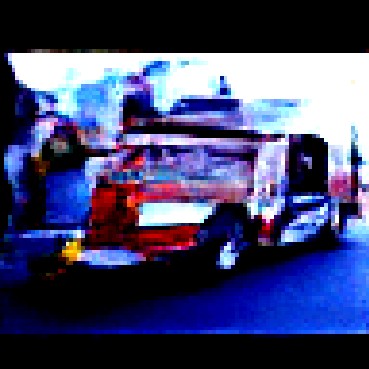} &
\includegraphics[width=2.25cm,valign=c]{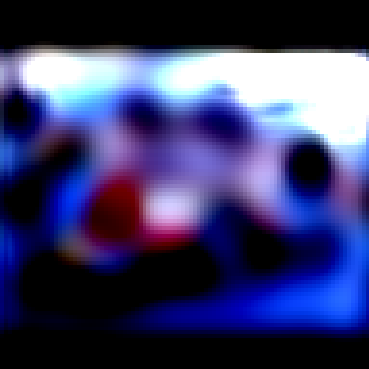} &
\includegraphics[width=2.25cm,valign=c]{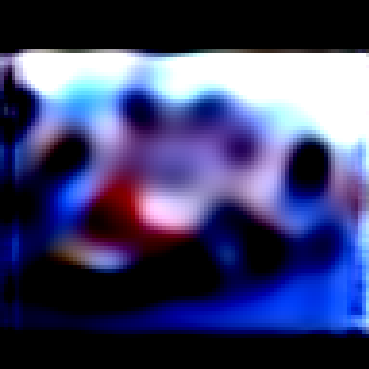}
\\ \rotatebox{90}{$s=7$}  &  \includegraphics[width=2.25cm,valign=c]{images/s1_img5_gt.png}  &
\includegraphics[width=2.25cm,valign=c]{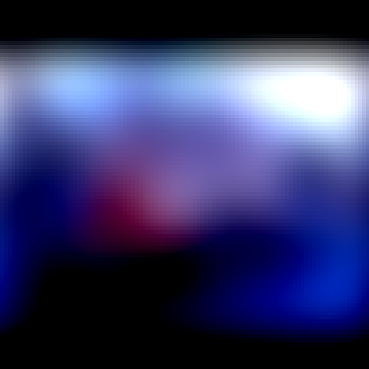} &
\includegraphics[width=2.25cm,valign=c]{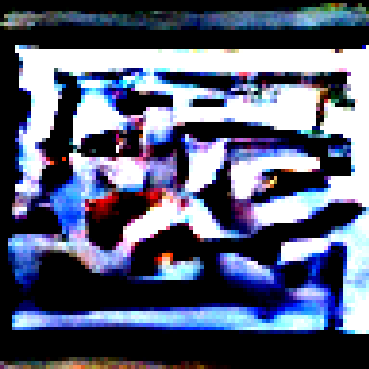} &
\includegraphics[width=2.25cm,valign=c]{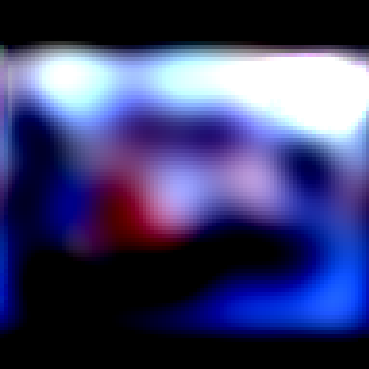} &
\includegraphics[width=2.25cm,valign=c]{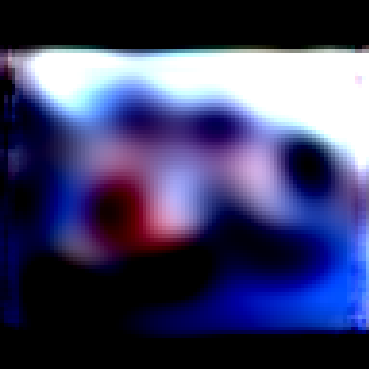}
\\ \rotatebox{90}{$s=9$} & \includegraphics[width=2.25cm,valign=c]{images/s1_img5_gt.png}  &
\includegraphics[width=2.25cm,valign=c]{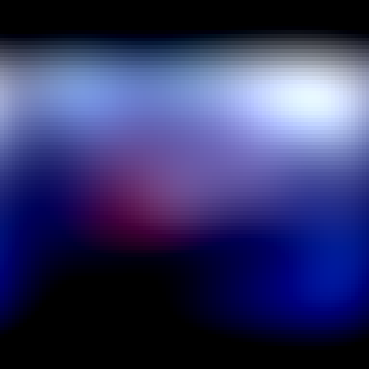} &
\includegraphics[width=2.25cm,valign=c]{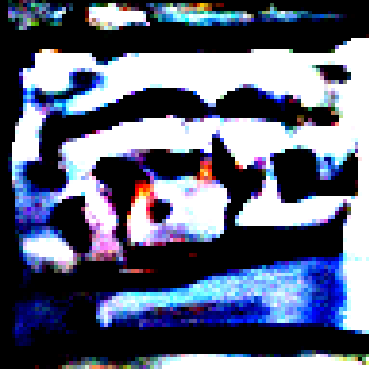} &
\includegraphics[width=2.25cm,valign=c]{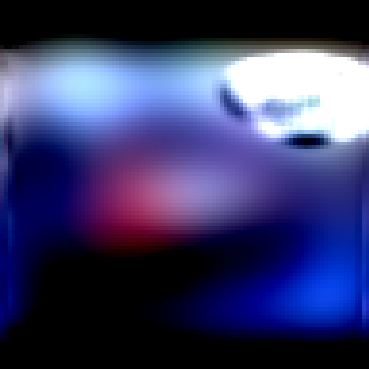} &
\includegraphics[width=2.25cm,valign=c]{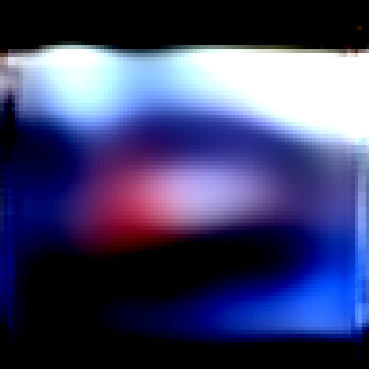}
\\ 
& (a)  & (b)  & (c) & (d) & (e) 
\end{tabular}}
\end{center}
\caption{An example of the recovery of deblurred images from the STL-10 data set. 
(a) Ground truth (same for all rows)
(b) Observed data,   
(c) Diffusion,  
(d) Proximal, (e) EUnet. Table \ref{tab2} reports numerical recovery results. Additional examples are provided in Appendix \ref{appendix:examples}.}
\label{fig:recovery_stl10}
\end{figure}

It is no surprise that our EUnet performs better than the Proximal method, because our method generalizes it. When comparing with diffusion models our method gives slightly worse results for problems where the blurring is small. Nonetheless, for problems where the blurring is significant, our method easily outperforms diffusion models. 

\subsection{Magnetics}

Observing that the real advantage of embedding is obtained for highly ill-posed problems we now turn our attention to such a problem.
Magnetics is a classical inverse problem in geophysical exploration \cite{parker}
and is commonly used to study the earth's interior and other planets \cite{mittelholz2022martian}.
The forward problem is given by Fredholm integral equation of the first kind
\begin{eqnarray}
    \label{magnetics}
    \bfb(\bfr') = \int_{\Omega} \left(\bfn_I \cdot \grad \grad (| \bfr - \bfr'|^{-1}) \cdot \bfn_J \right) \bfx(\bfr) d \bfr
\end{eqnarray}
where $\bfn_I$ and $\bfn_J$ are known direction vectors
and $\bfr$ and $\bfr'$ are location vectors. The model $\bfx(\bfr)$ is the magnetic susceptibility to be computed from the data $\bfb$. 
Typically, the data is measured on the top of the earth and one wishes to estimate the solution everywhere.
The magnetic problem is clearly highly ill-posed as we require the recovery of 2D solution from 1D data.

Upon discretization of the integral \Eqref{magnetics} using the midpoint method, we obtain a linear system. A sample from the images and the data which is a 1D vector that corresponds to the forward problem is presented in Figure~\ref{fig:model+data}.
We use the MNIST data set to train the system using our method as well as diffusion as proposed in \cite{chung2022diffusion}. 
The true images, images that are generated using the unrolled network EUnet, and images generated by the Diffusion model are presented in Figure~\ref{fig:model+data}.
The images clearly demonstrate that while the diffusion model fails to converge to an acceptable result, our EUnet yields plausible solutions.
\begin{figure}[h]
    \centering
    \begin{tabular}{cccccccccc}
     \includegraphics[width=1cm, height=1.2cm]{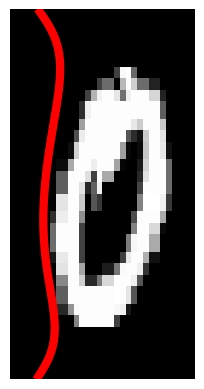}    & 
     \includegraphics[width=1cm, height=1.2cm]{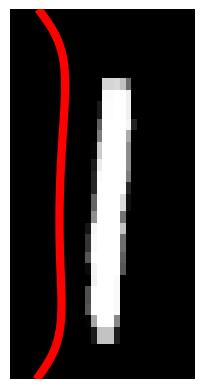}    & 
     \includegraphics[width=1cm, height=1.2cm]{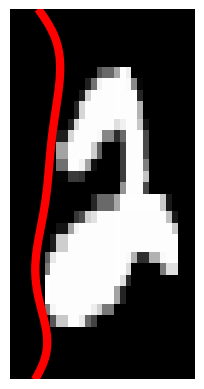}    & 
     \includegraphics[width=1cm, height=1.2cm]{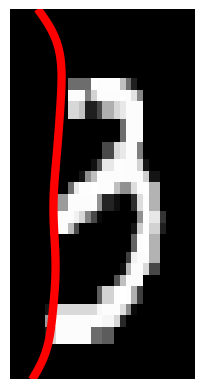}    & 
     \includegraphics[width=1cm, height=1.2cm]{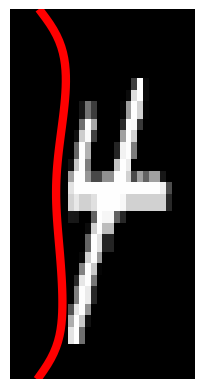}    &
     \includegraphics[width=1cm, height=1.2cm]{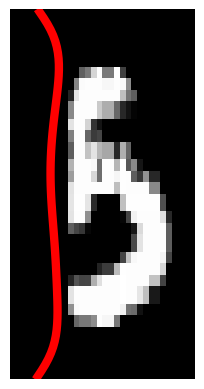}    & 
     \includegraphics[width=1cm, height=1.2cm]{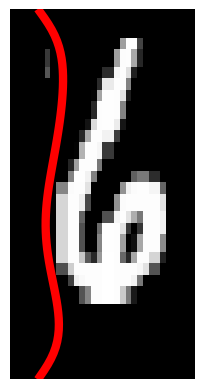}    & 
     \includegraphics[width=1cm, height=1.2cm]{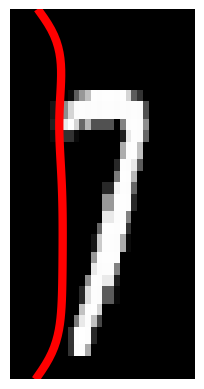}    & 
     \includegraphics[width=1cm, height=1.2cm]{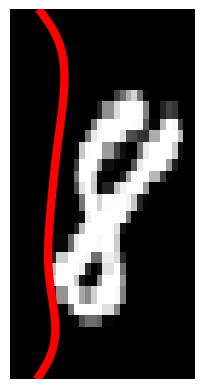}    & 
     \includegraphics[width=1cm, height=1.2cm]{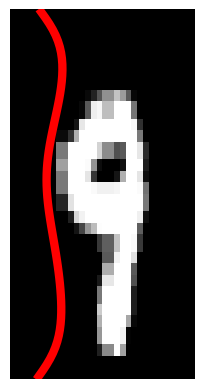}   
    \\
    \end{tabular}    
    \caption{Models and corresponding data (red line) for the inverse magnetics experiment.}
    \label{fig:model+data}
\end{figure}
\begin{figure}[h]
    \centering
    \resizebox{\linewidth}{!}{
    \begin{tabular}{cccc}
     \includegraphics[width=3.2cm, height=3.2cm]{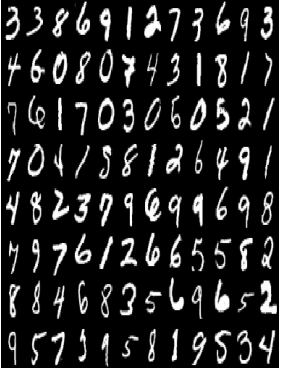}    & 
     \includegraphics[width=3.2cm, height=3.2cm]{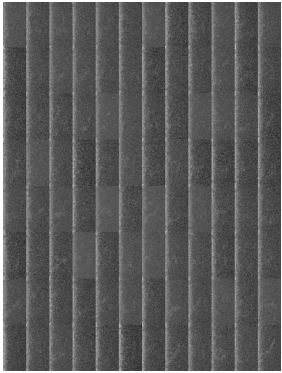}     &
     \includegraphics[width=3.2cm, height=3.2cm]{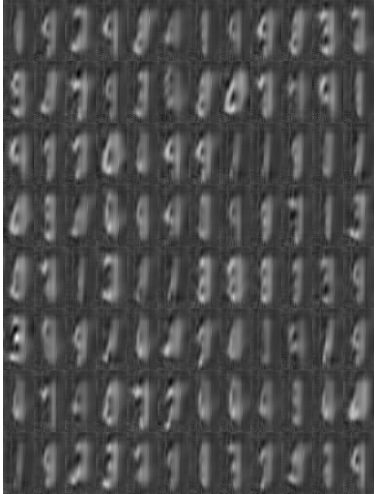}  
     & \includegraphics[width=3.2cm, height=3.2cm]{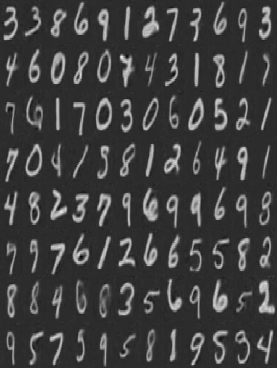}   
     \\
     True & Diffusion (MSE=9.1e-1) & Proximal (MSE=1.8e-1) & EUnet (MSE=6.2e-2)
    \end{tabular}}
    \caption{Recovery of MNIST images from magnetic data using Diffusion, Proximal, and EUnet.}
    \label{fig:enter-label}
\end{figure}
For the magnetic problem which is highly ill-posed, we observe that embedding is a key for a successful recovery of the solution. In Appendix \ref{appendix:examples}, we show the convergence of each of the methods, showing that our EUnet also offers faster convergence compared with proximal methods, as well as our optimization problem-based OPTEnet.

\section{Summary and conclusions}
\label{sec5}

In this paper, we have introduced a new method for inverse problems that incorporates learning an over-complete embedding of the solution as well as a regularization that acts on the hidden variables. Our method can be used either as an optimization problem or as an unrolled version. Our methods extend techniques that are based on over-complete dictionaries
such as basis pursuit by allowing to tailor a data-driven regularization for the basis and extend the regularization learning by changing the basis. 

We provide theoretical justification to the method and conduct experiments with a few model problems to demonstrate that indeed, there is merit in jointly learning the embedding as well as the regularization that acts on it.

Comparing our network to existing methods we observe that it outperforms other end-to-end techniques that do not embed the solution in high dimension. Moreover, our method significantly outperforms diffusion-based methods.

We believe that the main reason is that diffusion-based methods are not trained end-to-end and therefore may sample low-energy regions of the posterior. Incorporating an embedding in diffusion models and training them end-to-end is an open question that we believe can improve their performance for highly ill-posed problems.

\appendix

\section{Additional visualizations and convergence plot}
\label{appendix:examples}
\textbf{Additional deblurring visualizations.} now provide additional visualizations of the recovery quality of our EUnet compared with Proximal methods \cite{mardani2018neural} and Diffusion models \cite{chung2022diffusion}. As is evident from Table \ref{tab2}, our EUnet offers better recovery (lower MSE), and this improvement is reflected in the provided examples below.

 \begin{figure}[h]
\begin{center}
\begin{tabular}{c|c|c|c|c|c}
\centering \rotatebox{90}{$s=1$} & 
\includegraphics[width=2.25cm,valign=c]{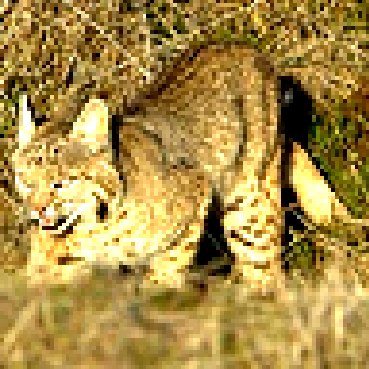}  &
\includegraphics[width=2.25cm,valign=c]{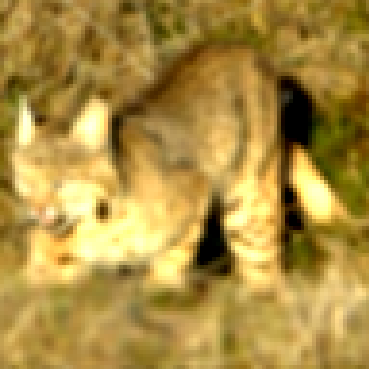} &
\includegraphics[width=2.3cm,valign=c]{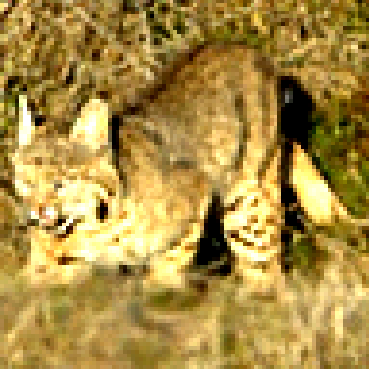}  &
\includegraphics[width=2.20cm,valign=c]{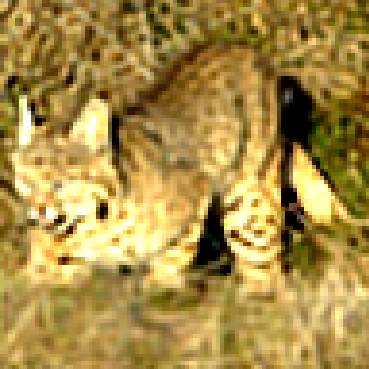} &
\includegraphics[width=2.25cm,valign=c]{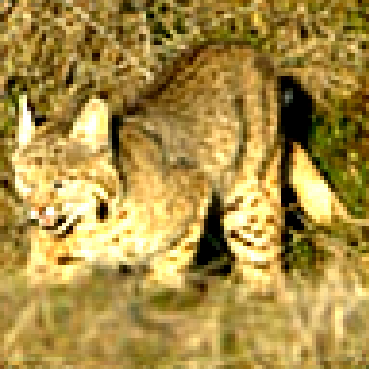}
\\ \rotatebox[origin=t]{90}{$s=3$} & \includegraphics[width=2.25cm,valign=c]{images/s1_img4_gt.png} &
\includegraphics[width=2.25cm,valign=c]{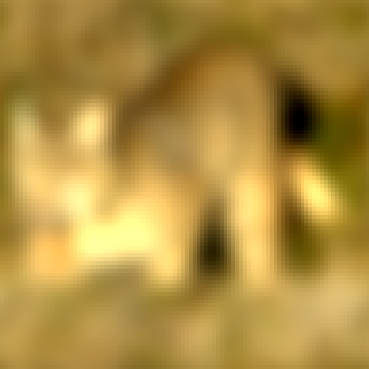} &
\includegraphics[width=2.25cm,valign=c]{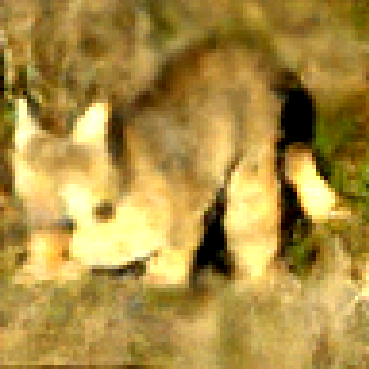} &
\includegraphics[width=2.25cm,valign=c]{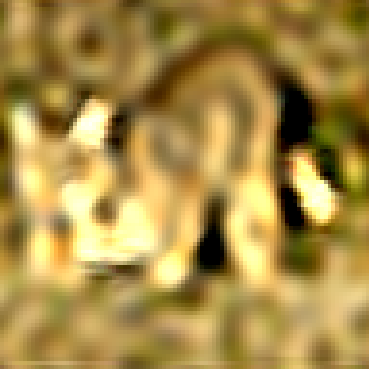} &
\includegraphics[width=2.25cm,valign=c]{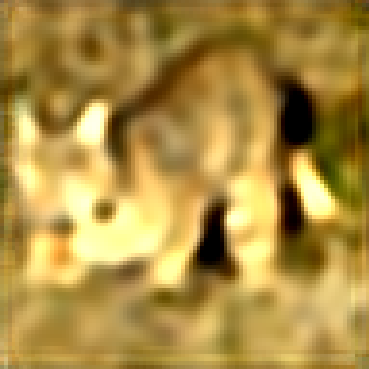}
\\ \rotatebox{90}{$s=5$} &  \includegraphics[width=2.25cm,valign=c]{images/s1_img4_gt.png}  &
\includegraphics[width=2.25cm,valign=c]{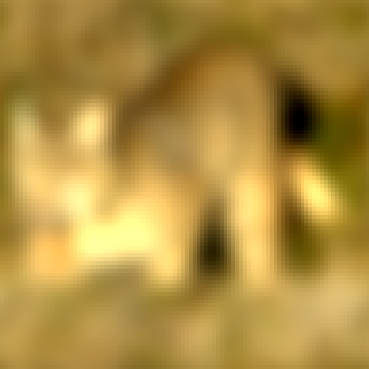} &
\includegraphics[width=2.25cm,valign=c]{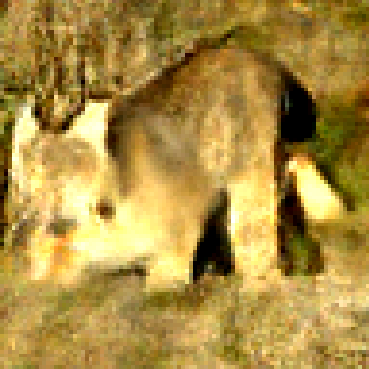} &
\includegraphics[width=2.25cm,valign=c]{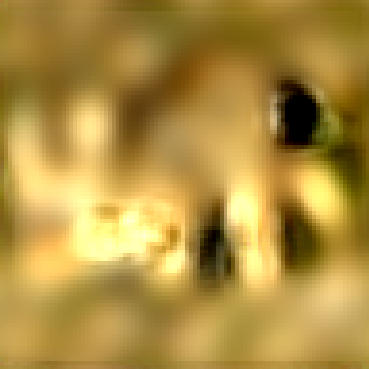} &
\includegraphics[width=2.25cm,valign=c]{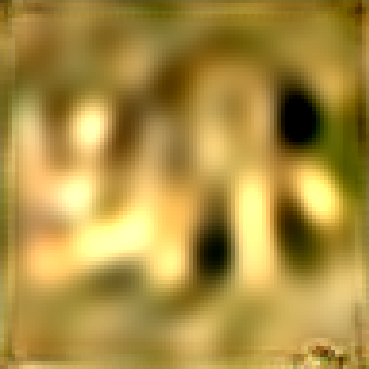}
\\ \rotatebox{90}{$s=7$}  &  \includegraphics[width=2.25cm,valign=c]{images/s1_img4_gt.png}  &
\includegraphics[width=2.25cm,valign=c]{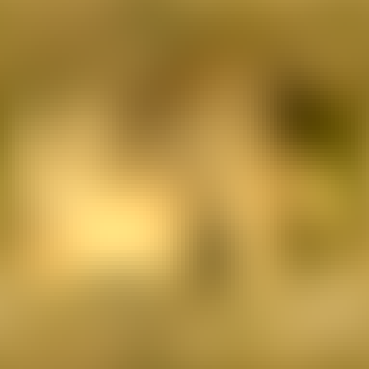} &
\includegraphics[width=2.25cm,valign=c]{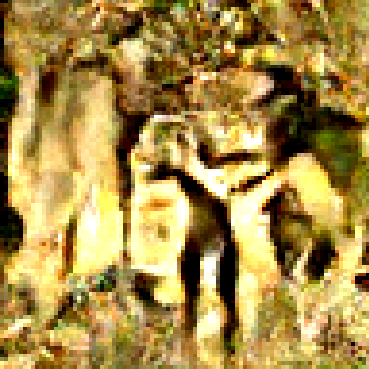} &
\includegraphics[width=2.25cm,valign=c]{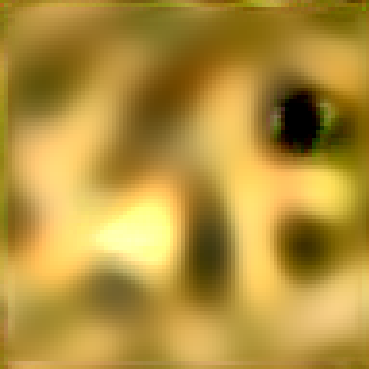} &
\includegraphics[width=2.25cm,valign=c]{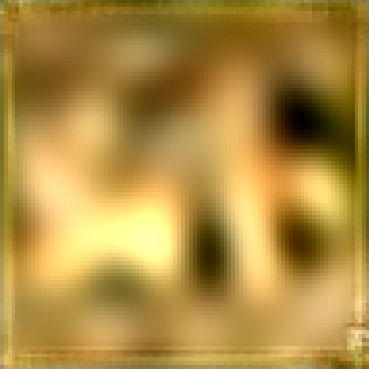}
\\ \rotatebox{90}{$s=9$} & \includegraphics[width=2.25cm,valign=c]{images/s1_img4_gt.png}  &
\includegraphics[width=2.25cm,valign=c]{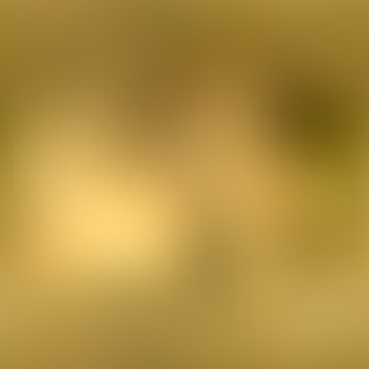} &
\includegraphics[width=2.25cm,valign=c]{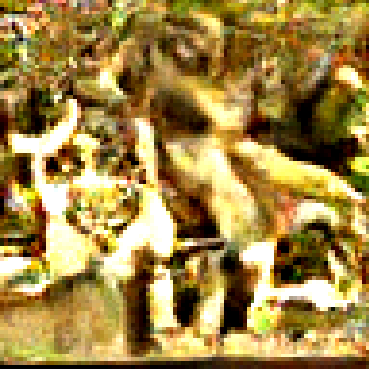} &
\includegraphics[width=2.25cm,valign=c]{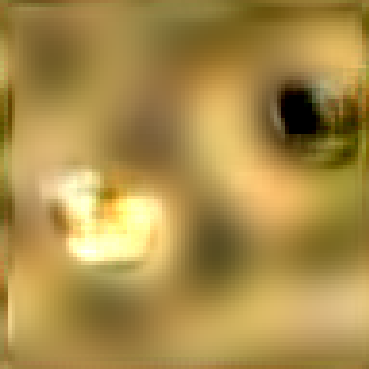} &
\includegraphics[width=2.25cm,valign=c]{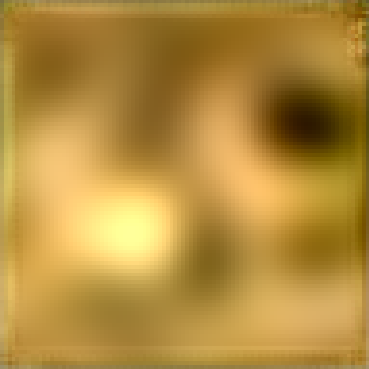}
\\ 
& (a)  & (b)  & (c) & (d) & (e) 
\end{tabular}
\end{center}
\caption{An additional example of the recovery of deblurred images from the STL-10 data set. 
(a) Ground truth (same for all rows)
(b) Observed data,   
(c) Diffusion,  
(d) Proximal, (e) EUnet. Table \ref{tab2} reports numerical recovery results.}
\label{fig:additional4}
\end{figure}

 \begin{figure}[t]
\begin{center}
\begin{tabular}{c|c|c|c|c|c}
\centering \rotatebox{90}{$s=1$} & 
\includegraphics[width=2.25cm,valign=c]{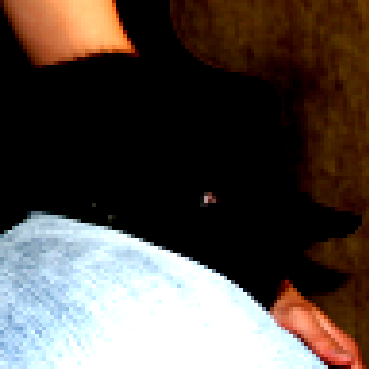}  &
\includegraphics[width=2.25cm,valign=c]{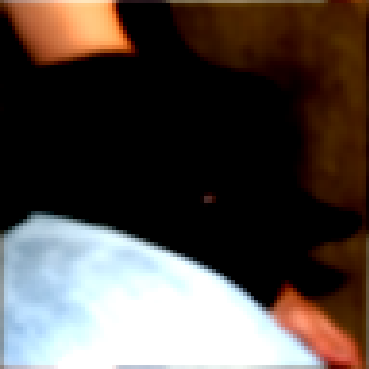} &
\includegraphics[width=2.3cm,valign=c]{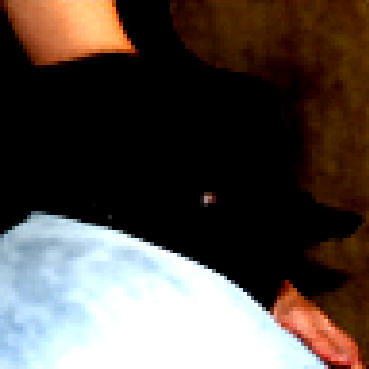}  &
\includegraphics[width=2.20cm,valign=c]{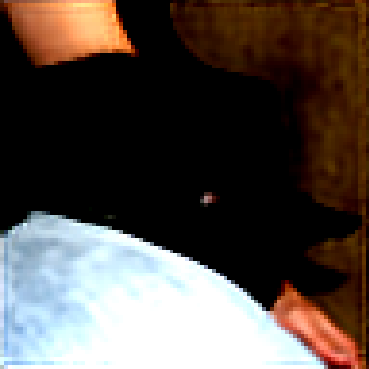} &
\includegraphics[width=2.25cm,valign=c]{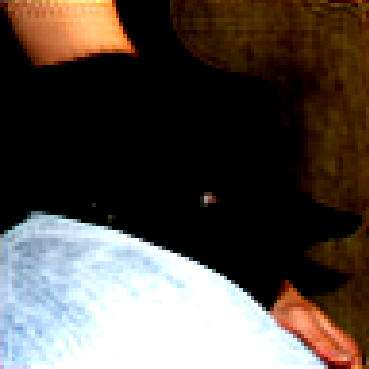}
\\ \rotatebox[origin=t]{90}{$s=3$} & \includegraphics[width=2.25cm,valign=c]{images/s1_img8_gt.png} &
\includegraphics[width=2.25cm,valign=c]{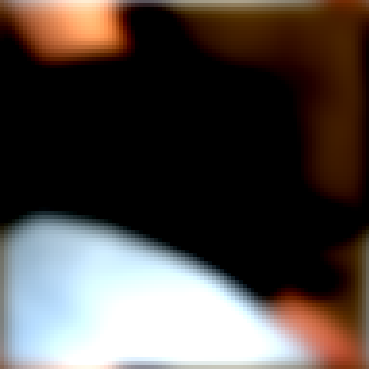} &
\includegraphics[width=2.25cm,valign=c]{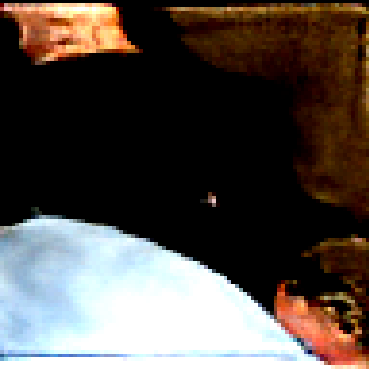} &
\includegraphics[width=2.25cm,valign=c]{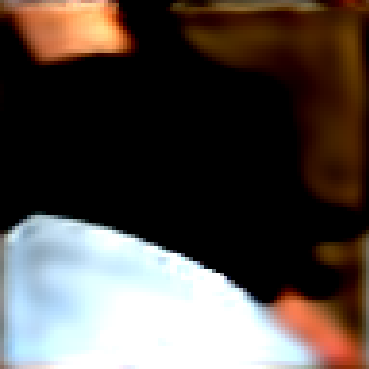} &
\includegraphics[width=2.25cm,valign=c]{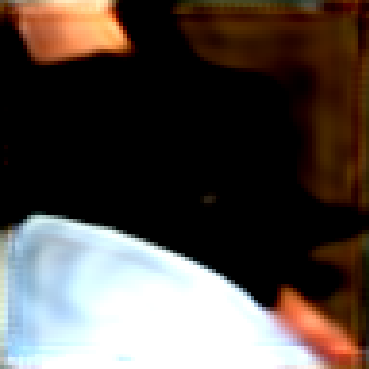}
\\ \rotatebox{90}{$s=5$} &  \includegraphics[width=2.25cm,valign=c]{images/s1_img8_gt.png}  &
\includegraphics[width=2.25cm,valign=c]{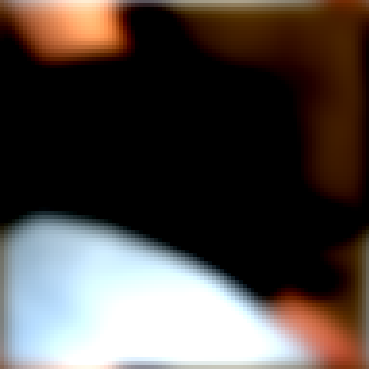} &
\includegraphics[width=2.25cm,valign=c]{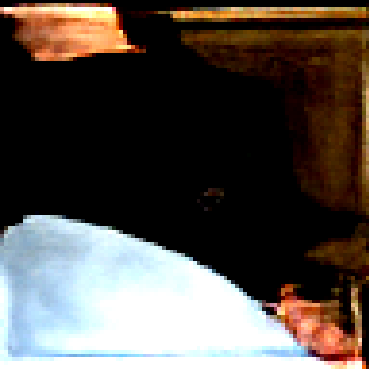} &
\includegraphics[width=2.25cm,valign=c]{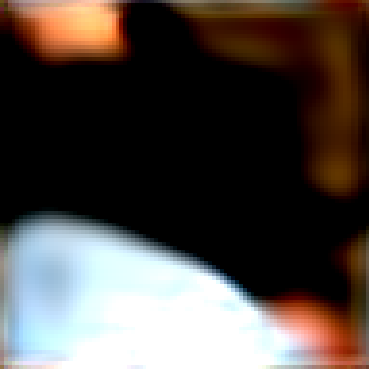} &
\includegraphics[width=2.25cm,valign=c]{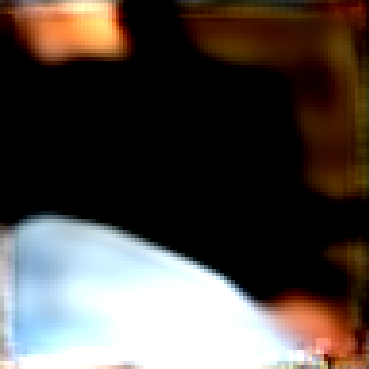}
\\ \rotatebox{90}{$s=7$}  &  \includegraphics[width=2.25cm,valign=c]{images/s1_img8_gt.png}  &
\includegraphics[width=2.25cm,valign=c]{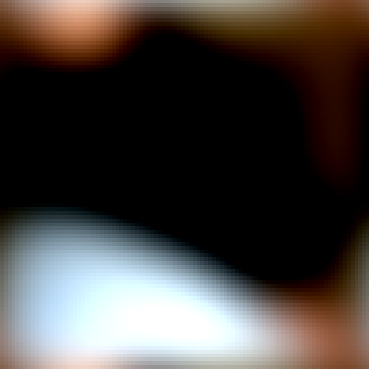} &
\includegraphics[width=2.25cm,valign=c]{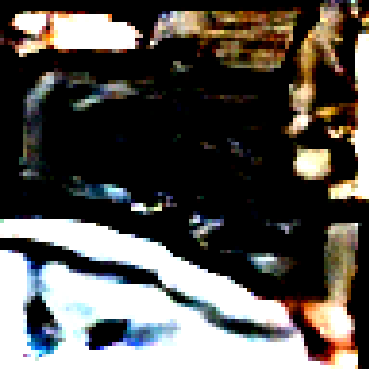} &
\includegraphics[width=2.25cm,valign=c]{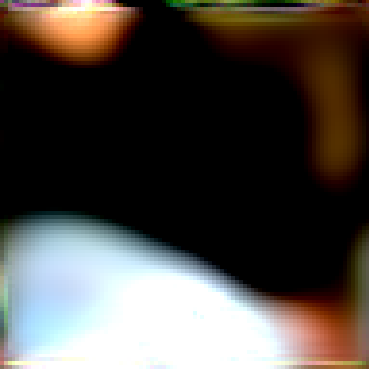} &
\includegraphics[width=2.25cm,valign=c]{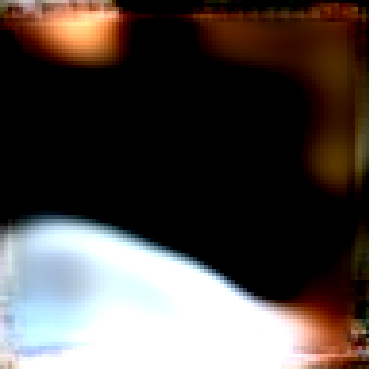}
\\ \rotatebox{90}{$s=9$} & \includegraphics[width=2.25cm,valign=c]{images/s1_img8_gt.png}  &
\includegraphics[width=2.25cm,valign=c]{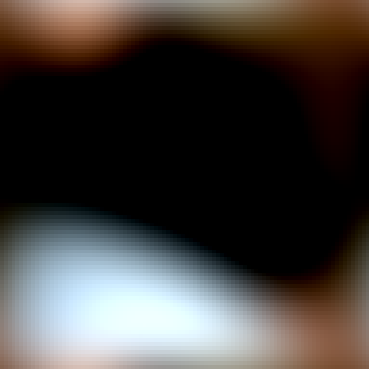} &
\includegraphics[width=2.25cm,valign=c]{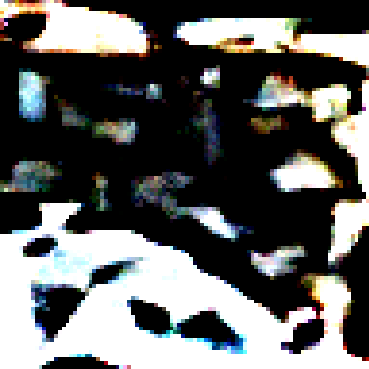} &
\includegraphics[width=2.25cm,valign=c]{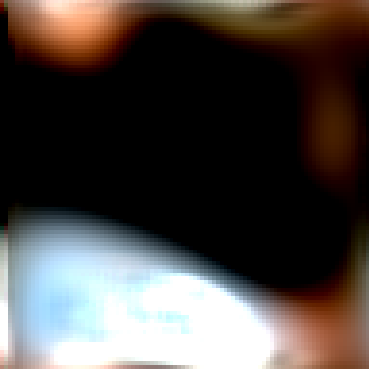} &
\includegraphics[width=2.25cm,valign=c]{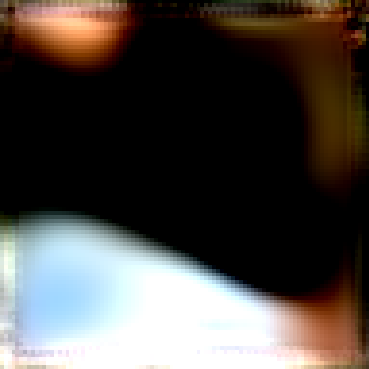}
\\ 
& (a)  & (b)  & (c) & (d) & (e) 
\end{tabular}
\end{center}
\caption{Additional example of the recovery of deblurred images from the STL-10 data set. 
(a) Ground truth (same for all rows)
(b) Observed data,   
(c) Diffusion,  
(d) Proximal, (e) EUnet. Table \ref{tab2} reports numerical recovery results.}
\label{fig:additional5}
\end{figure}

 \begin{figure}[t]
\begin{center}
\begin{tabular}{c|c|c|c|c|c}
\centering \rotatebox{90}{$s=1$} & 
\includegraphics[width=2.25cm,valign=c]{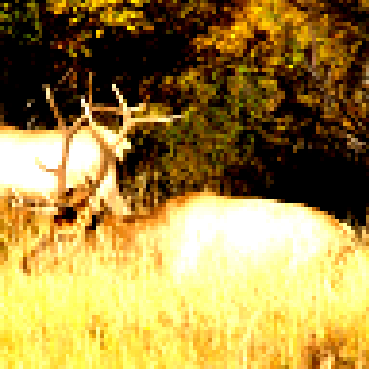}  &
\includegraphics[width=2.25cm,valign=c]{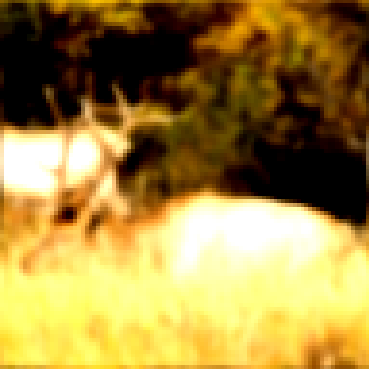} &
\includegraphics[width=2.3cm,valign=c]{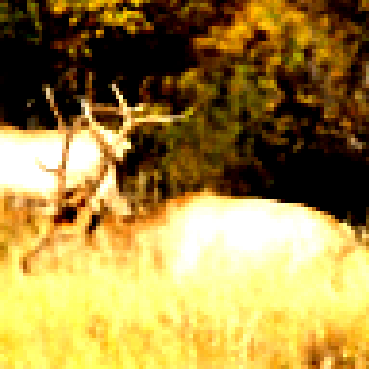}  &
\includegraphics[width=2.20cm,valign=c]{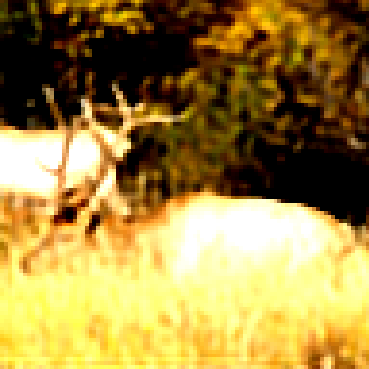} &
\includegraphics[width=2.25cm,valign=c]{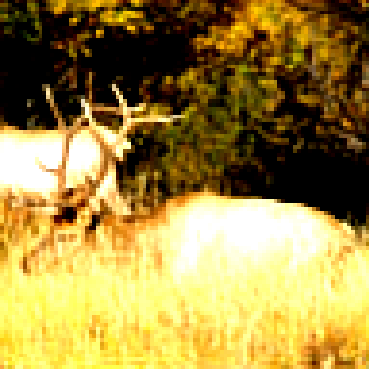}
\\ \rotatebox[origin=t]{90}{$s=3$} & \includegraphics[width=2.25cm,valign=c]{images/s1_img7_gt.png} &
\includegraphics[width=2.25cm,valign=c]{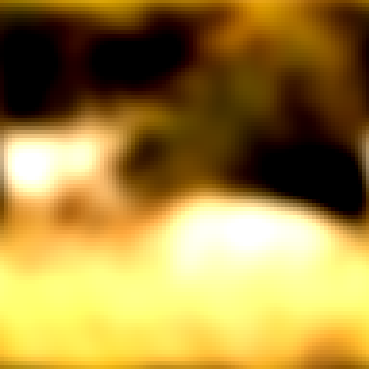} &
\includegraphics[width=2.25cm,valign=c]{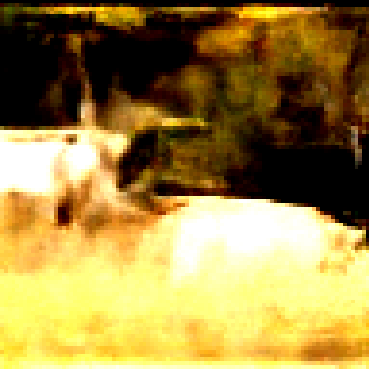} &
\includegraphics[width=2.25cm,valign=c]{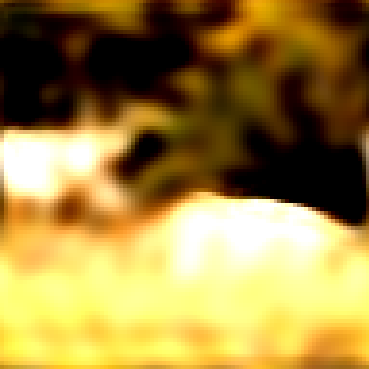} &
\includegraphics[width=2.25cm,valign=c]{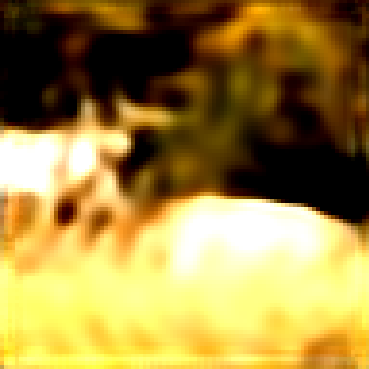}
\\ \rotatebox{90}{$s=5$} &  \includegraphics[width=2.25cm,valign=c]{images/s1_img7_gt.png}  &
\includegraphics[width=2.25cm,valign=c]{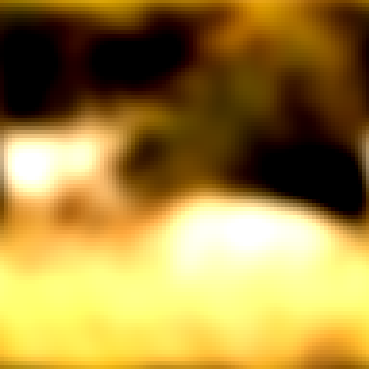} &
\includegraphics[width=2.25cm,valign=c]{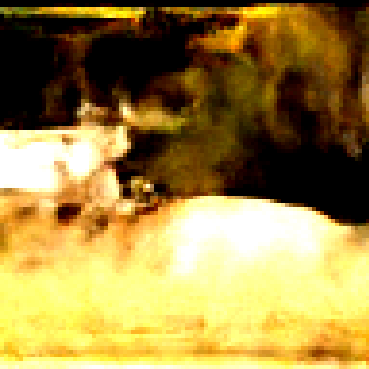} &
\includegraphics[width=2.25cm,valign=c]{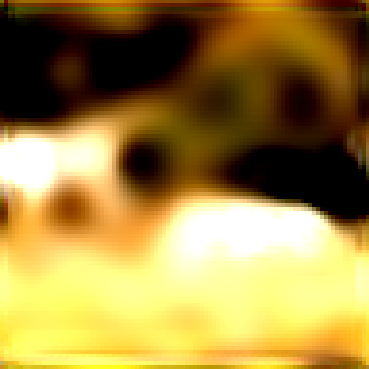} &
\includegraphics[width=2.25cm,valign=c]{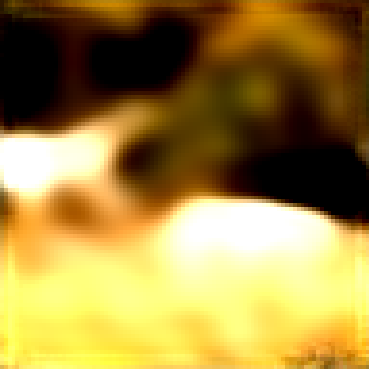}
\\ \rotatebox{90}{$s=7$}  &  \includegraphics[width=2.25cm,valign=c]{images/s1_img7_gt.png}  &
\includegraphics[width=2.25cm,valign=c]{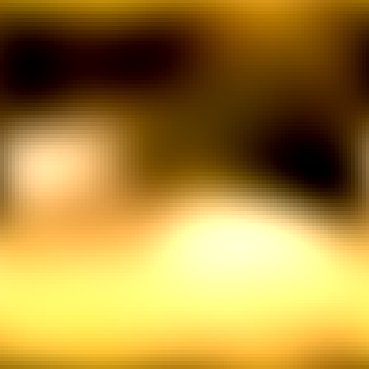} &
\includegraphics[width=2.25cm,valign=c]{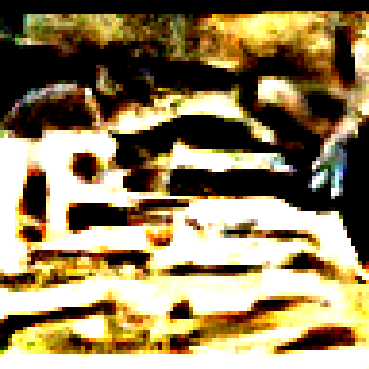} &
\includegraphics[width=2.25cm,valign=c]{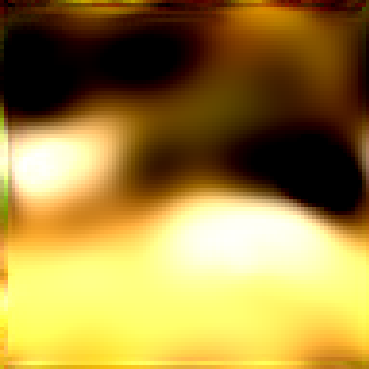} &
\includegraphics[width=2.25cm,valign=c]{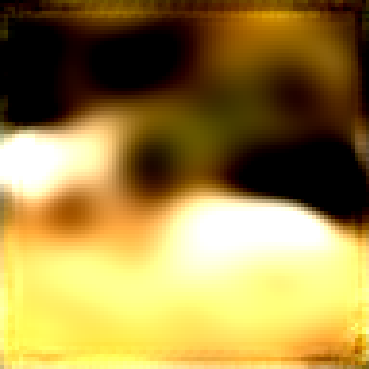}
\\ \rotatebox{90}{$s=9$} & \includegraphics[width=2.25cm,valign=c]{images/s1_img7_gt.png}  &
\includegraphics[width=2.25cm,valign=c]{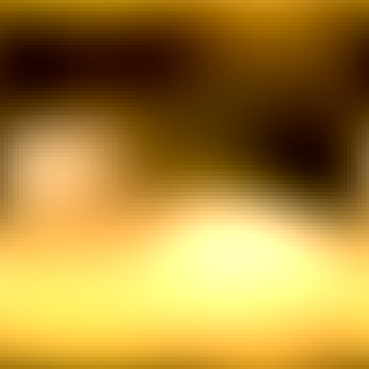} &
\includegraphics[width=2.25cm,valign=c]{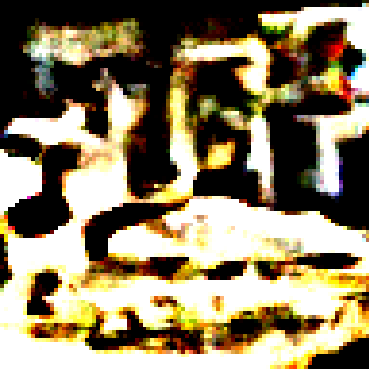} &
\includegraphics[width=2.25cm,valign=c]{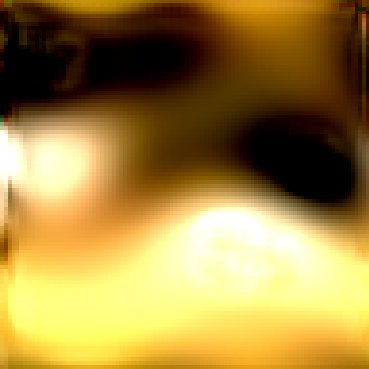} &
\includegraphics[width=2.25cm,valign=c]{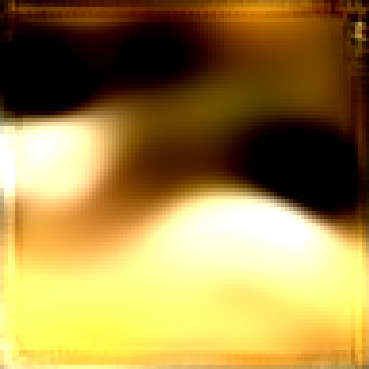}
\\ 
& (a)  & (b)  & (c) & (d) & (e) 
\end{tabular}
\end{center}
\caption{An additional example of the recovery of deblurred  images from the STL-10 data set. 
(a) Ground truth (same for all rows)
(b) Observed data,   
(c) Diffusion,  
(d) Proximal, (e) EUnet. Table \ref{tab2} reports numerical recovery results.}
\label{fig:additional2}
\end{figure}

 \begin{figure}[t]
\begin{center}
\begin{tabular}{c|c|c|c|c|c}
\centering \rotatebox{90}{$s=1$} & 
\includegraphics[width=2.25cm,valign=c]{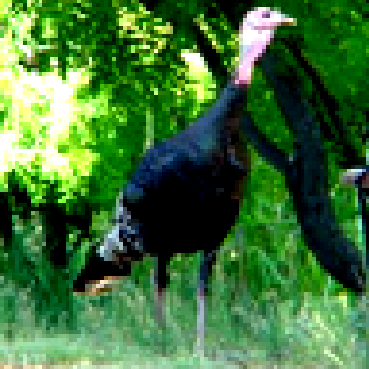}  &
\includegraphics[width=2.25cm,valign=c]{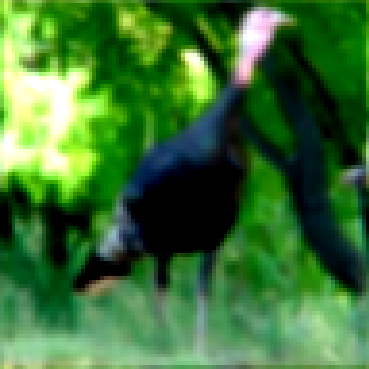} &
\includegraphics[width=2.3cm,valign=c]{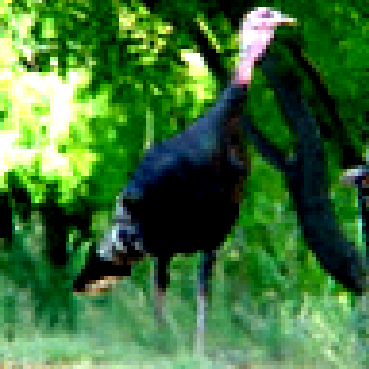}  &
\includegraphics[width=2.20cm,valign=c]{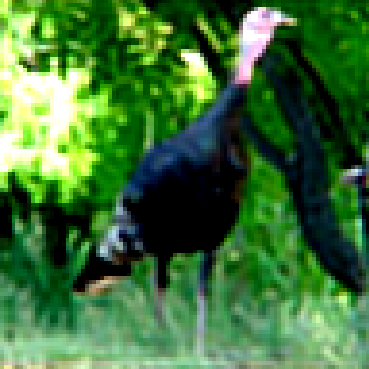} &
\includegraphics[width=2.25cm,valign=c]{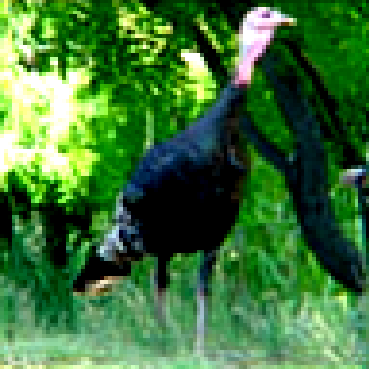}
\\ \rotatebox[origin=t]{90}{$s=3$} & \includegraphics[width=2.25cm,valign=c]{images/s1_img2_gt.png} &
\includegraphics[width=2.25cm,valign=c]{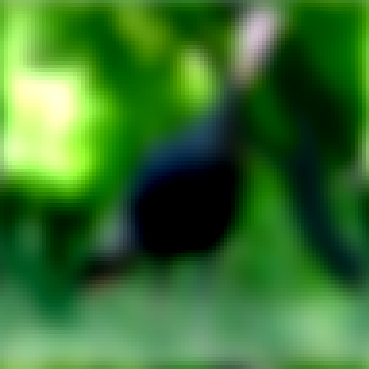} &
\includegraphics[width=2.25cm,valign=c]{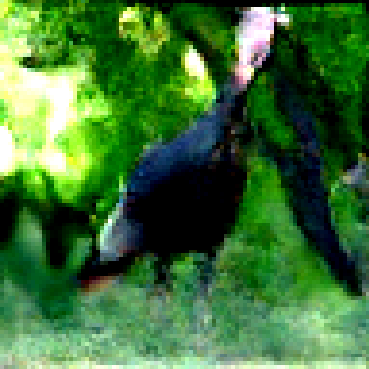} &
\includegraphics[width=2.25cm,valign=c]{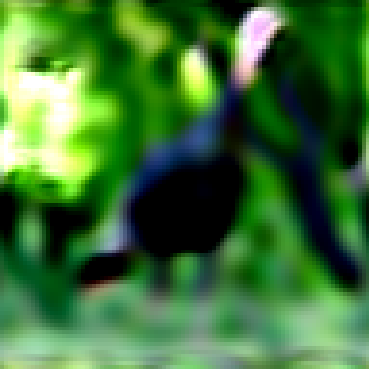} &
\includegraphics[width=2.25cm,valign=c]{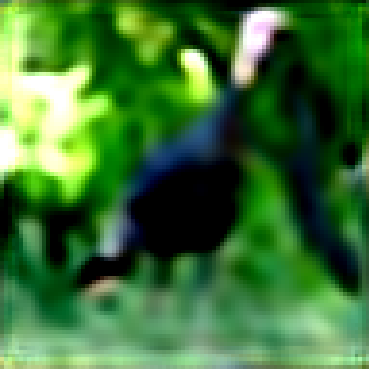}
\\ \rotatebox{90}{$s=5$} &  \includegraphics[width=2.25cm,valign=c]{images/s1_img2_gt.png}  &
\includegraphics[width=2.25cm,valign=c]{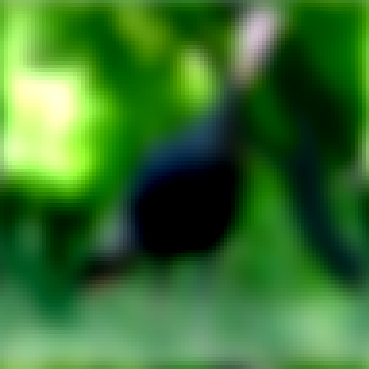} &
\includegraphics[width=2.25cm,valign=c]{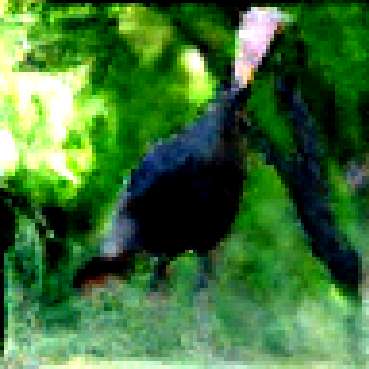} &
\includegraphics[width=2.25cm,valign=c]{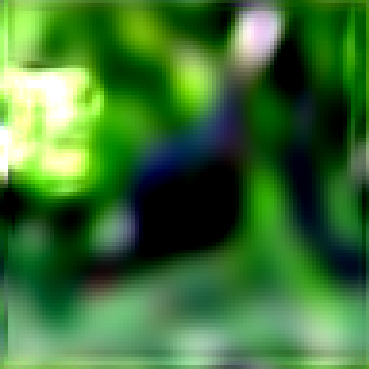} &
\includegraphics[width=2.25cm,valign=c]{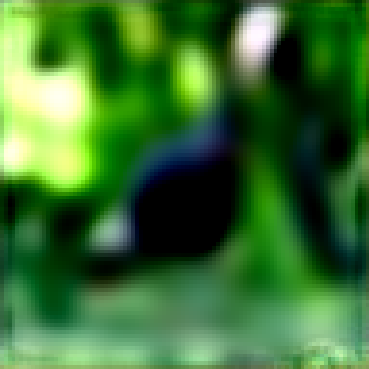}
\\ \rotatebox{90}{$s=7$}  &  \includegraphics[width=2.25cm,valign=c]{images/s1_img2_gt.png}  &
\includegraphics[width=2.25cm,valign=c]{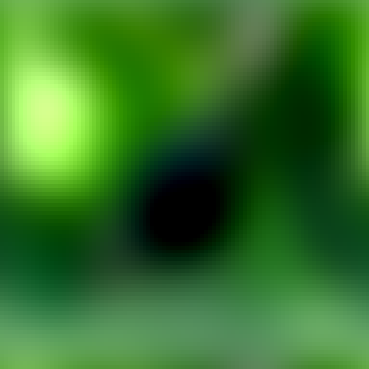} &
\includegraphics[width=2.25cm,valign=c]{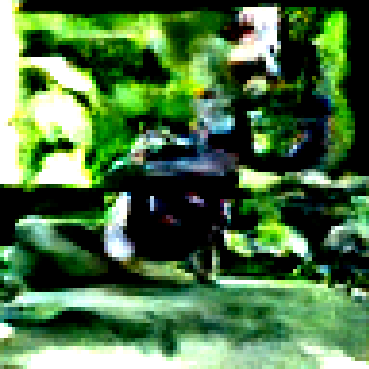} &
\includegraphics[width=2.25cm,valign=c]{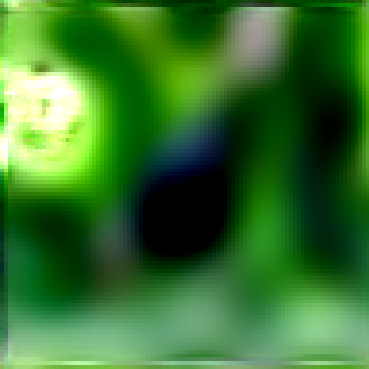} &
\includegraphics[width=2.25cm,valign=c]{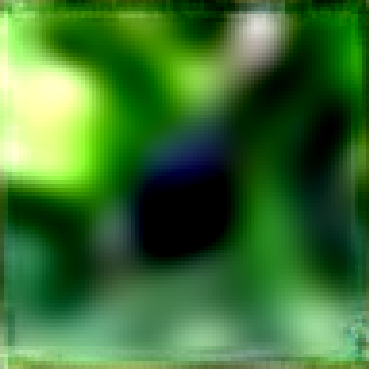}
\\ \rotatebox{90}{$s=9$} & \includegraphics[width=2.25cm,valign=c]{images/s1_img2_gt.png}  &
\includegraphics[width=2.25cm,valign=c]{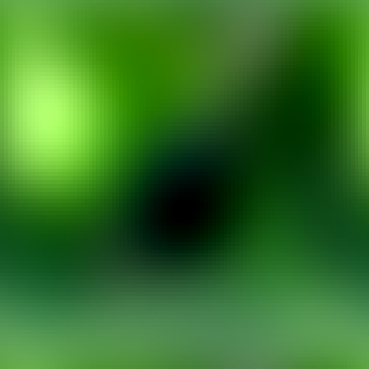} &
\includegraphics[width=2.25cm,valign=c]{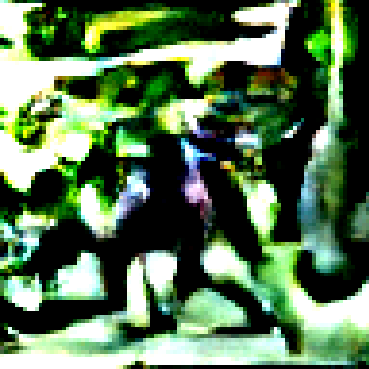} &
\includegraphics[width=2.25cm,valign=c]{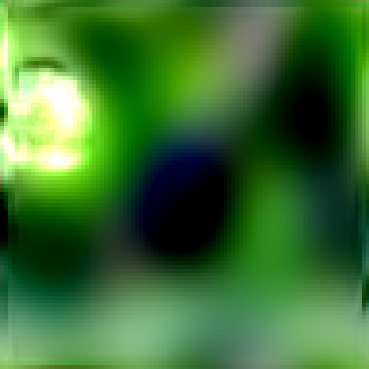} &
\includegraphics[width=2.25cm,valign=c]{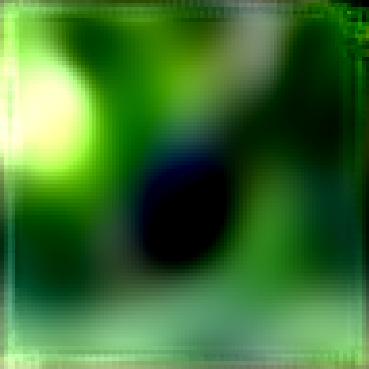}
\\ 
& (a)  & (b)  & (c) & (d) & (e) 
\end{tabular}
\end{center}
\caption{Additional example of the recovery of deblurred images from the STL-10 data set. 
(a) Ground truth (same for all rows)
(b) Observed data,   
(c) Diffusion,  
(d) Proximal, (e) EUnet. Table \ref{tab2} reports numerical recovery results.}
\label{fig:additional3}
\end{figure}

\clearpage

\textbf{Convergence of EUnet.}
We report the convergence plots of the Proximal \cite{mardani2018neural} method and our OPTEnet and EUnet, in Figure~\ref{fig:magconv}. We observe that our EUnet offers faster convergence, in addition to improved performance in terms of recovery.
\begin{figure}[h]
    \centering
    \includegraphics[width=0.9\linewidth, height=0.7\linewidth]{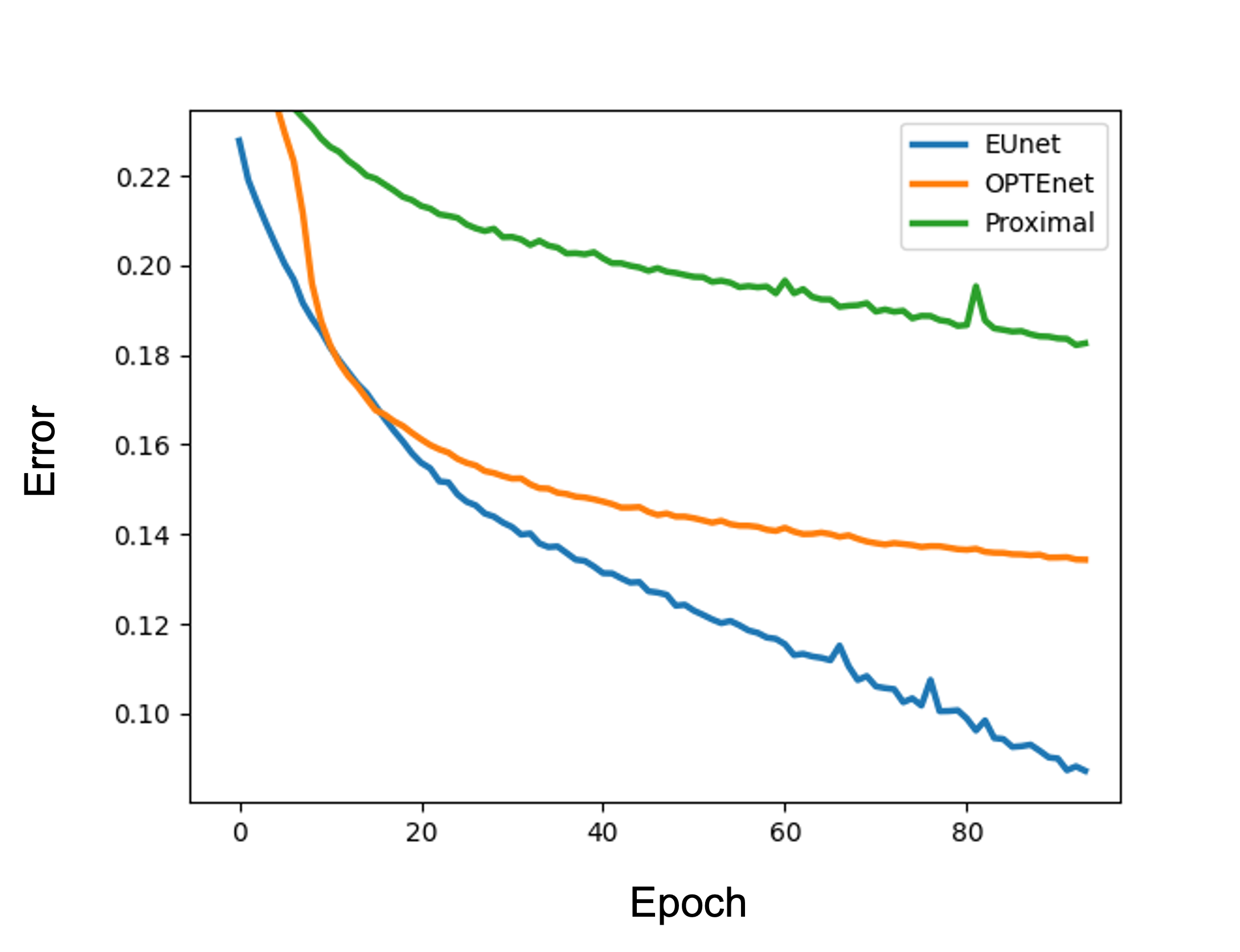}
    \caption{Convergence history for different methods in the magnetics experiment. \label{fig:magconv}}
\end{figure}

\bibliographystyle{plain}

\bibliography{iclr2021_conference}

\end{document}